\documentclass{article}

\usepackage[top=1in, bottom=1in, left=1in, right=1in]{geometry}
\usepackage[numbers,compress,sort]{natbib}
\usepackage{amsmath,amsfonts,amssymb,amsthm}
\usepackage[utf8]{inputenc} 
\usepackage[T1]{fontenc} 
\usepackage{hyperref} 
\usepackage{url} 
\usepackage{booktabs} 
\usepackage{microtype}  
\usepackage{xcolor} 
\usepackage{times}
\usepackage{textcomp}
\usepackage{dsfont}
\usepackage{cleveref}
\usepackage{enumitem}
\usepackage{bm}
\usepackage{bbm}
\usepackage{soul}
\usepackage{algorithm}
\usepackage[noend]{algorithmic}
\usepackage{outlines}
\usepackage{subcaption}
\usepackage{xspace}
\usepackage{thmtools}
\usepackage{thm-restate}
\usepackage{wrapfig}
\usepackage[outdir=./]{epstopdf}
\usepackage{multirow}
\usepackage{array}
\usepackage{adjustbox}
\usepackage{float}

\DeclareCaptionLabelFormat{fullfig}{Figure~\thefigure(\alph{subfigure})}

\allowdisplaybreaks

\newif\ifcomment
\commenttrue %

\newif\ifrevise
\revisetrue

\AtBeginEnvironment{definition}{}

\theoremstyle{plain}
\newtheorem{assumption}{Assumption}
\newtheorem{theorem}{Theorem}
\newtheorem{proposition}{Proposition}
\newtheorem{lemma}{Lemma}

\theoremstyle{remark}
\newtheorem{remark}{\textit{Remark}}
\newtheorem*{remark*}{Remark}

\newcommand{\norm}[1]{\left\lVert#1\right\rVert}

\newcommand{\normbig}[1]{\big\lVert#1\big\rVert}

\newcommand{\E}[1]{\mathbb{E}\left[#1\right]}

\newcommand{\givenplain}{\;|\;}

\newcommand{\givenmiddle}{\;\middle|\;}

\newcommand{\R}{\mathbb{R}}

\newcommand{\abs}[1]{\left\lvert#1\right\rvert}

\newcommand{\absbigg}[1]{\bigg\lvert#1\bigg\rvert}

\newcommand{\mx}{\bm{x}}

\newcommand{\sspa}{\mathbb{S}}
\newcommand{\aspa}{\mathbb{A}}

\newcommand{\sumN}{\sum_{i\in[N]}}
\newcommand{\sumell}{\sum_{\ell=0}^\infty}
\newcommand{\supell}{\sup_{\ell\in\mathbb{N}}}

\newcommand{\sumk}{\sum_{k\in[K]}}
\newcommand{\maxk}{\max_{k\in[K]}}
\newcommand{\maxmp}{\max_{m'\in[0,1]_N\colon m'\leq m}}

\newcommand{\rel}{\textup{rel}} %
\newcommand{\rmax}{r_{\max}}
\newcommand{\cmax}{c_{\max}}

\newcommand{\ravg}{R}

\newcommand{\rrel}{\ravg^\rel}

\newcommand{\pibar}{{\sysbar{\pi}}}
\newcommand{\pibs}{{\pibar^*}}

\newcommand{\sysbar}[1]{\bar{#1}}

\newcommand{\md}{m_d}
\newcommand{\Md}[1]{N\md(x)}

\newcommand{\rhov}{\rho_V}
\newcommand{\Kv}{K_V}
\newcommand{\Ctau}{C_\tau}
\newcommand{\tmix}{\tau}

\newcommand{\hid}{h_{\textnormal{ID}}}

\newcommand{\ID}{\textnormal{ID}}
\newcommand{\finite}{\textnormal{finite}}
\newcommand{\newID}{\text{newID}}
\newcommand{\cardA}{\lvert\aspa\rvert}
\newcommand{\cardS}{\lvert\sspa\rvert}

\newcommand{\conf}{\text{conf}}
\newcommand{\mono}{\text{mono}}
\newcommand{\cov}{\text{cov}}

\newcommand{\rembud}{\overline{C}^*} %
\newcommand{\costIDreassign}{\delta} %
\newcommand{\gsizeIDreassign}{d} %

\usepackage{algorithm}
\usepackage{algorithmic}
\usepackage{amsmath,amsfonts,bm}
\usepackage{thm-restate}

\def\mX{{\bm{X}}}

\DeclareMathAlphabet{\mathsfit}{\encodingdefault}{\sfdefault}{m}{sl}
\SetMathAlphabet{\mathsfit}{bold}{\encodingdefault}{\sfdefault}{bx}{n}

\newcommand{\EE}{\mathbb{E}}
\newcommand{\RR}{\mathbb{R}}

\newcommand{\cA}{\mathcal{A}}

\newcommand{\cF}{\mathcal{F}}

\newcommand{\cI}{\mathcal{I}}

\newcommand{\cD}{\mathcal{D}}

\newcommand{\innerproduct}[2]{\left\langle #1, #2 \right\rangle}
\newcommand{\cM}{\mathcal{M}}

\newcommand{\bracket}[1]{\left(#1\right)}
\newcommand{\mbracket}[1]{\left[#1\right]}
\newcommand{\sets}[1]{\left\{ #1 \right\}}
\newcommand{\PP}{\mathbb{P} }

\newif\ifsup\supfalse
\suptrue

\DeclareMathOperator*{\argmax}{argmax}

\title{Projection-based Lyapunov method for fully heterogeneous weakly-coupled MDPs}

\renewcommand{\thefootnote}{\fnsymbol{footnote}}

\author{%
Xiangcheng Zhang$^{1}\footnotemark[1] ~~\footnotemark[2]$ \quad Yige Hong$^{2}\footnotemark[1]$ \quad  Weina Wang$^{2}$\medskip\\
$^1$Weiyang College, Tsinghua University  \\ $^2$ Computer Science Department, Carnegie Mellon University \\
\texttt{xc-zhang21@mails.tsinghua.edu.cn}\\
\texttt{\{yigeh,weinaw\}@cs.cmu.edu}
}

\begin{document}

\footnotetext[1]{These authors contributed equally to this work.}
\footnotetext[2]{Work done during a visit at Carnegie Mellon University.}
\renewcommand{\thefootnote}{\arabic{footnote}}
\setcounter{footnote}{0}

\renewcommand{\thefootnote}{\arabic{footnote}}

\date{}
\maketitle

\begin{abstract}%
Heterogeneity poses a fundamental challenge for many real-world large-scale decision-making problems but remains largely understudied.
In this paper, we study the \emph{fully heterogeneous} setting of a prominent class of such problems, known as weakly-coupled Markov decision processes (WCMDPs).
Each WCMDP consists of $N$ arms (or subproblems), which have distinct model parameters in the fully heterogeneous setting, leading to the curse of dimensionality when $N$ is large.
We show that, under mild assumptions, an efficiently computable policy achieves an $O(1/\sqrt{N})$ optimality gap in the long-run average reward per arm for fully heterogeneous WCMDPs as $N$ becomes large.
This is the \emph{first asymptotic optimality result} for fully heterogeneous average-reward WCMDPs.
Our main technical innovation is the construction of projection-based Lyapunov functions that certify the convergence of rewards and costs to an optimal region, even under full heterogeneity.
\end{abstract}

\tableofcontents

\section{Introduction}
\label{sec:introduction}

Heterogeneity poses a fundamental challenge for many real-world decision-making problems, where each problem consists of a large number of interacting components.
However, despite its practical significance, heterogeneity remains largely understudied in the literature.
In this paper, we study \emph{heterogeneous} settings of a prominent class of such problems, known as weakly-coupled Markov decision processes (WCMDPs) \citep{Haw_03}.
A WCMDP consists of $N$ \emph{arms} (or \emph{subproblems}), where each arm itself is a Markov decision process (MDP).
In a heterogeneous setting, the MDPs could be distinct. 
At each time step, the decision-maker selects an action for each arm, which affects the arm's transition probabilities and reward, and then the arms make state transitions independently.
However, these actions are subject to a set of global \emph{budget constraints}, where each constraint limits one type of total cost across all arms at each time step.
The objective is to find a policy that maximizes the long-run \emph{average reward} over an infinite time horizon.
We focus on the \emph{planning} setting, where all the model parameters (reward function, cost functions, budget, and transition kernel) are known. \nocite{ZhaHonWan_25_het_wcmdp_nips_ver_on_arxiv}

WCMDPs have been used to model a wide range of applications, including online advertising \citep{boutilier2016budget,zhou2023rl}, job scheduling \citep{yu2018deadline}, healthcare \citep{biswas2021learning}, surveillance \citep{villar2016indexability}, and machine maintenance \citep{GlaMitAns_05_rb_repair}.
A faithful modeling of these applications calls for \emph{heterogeneity}.
For instance, in \citep{biswas2021learning}, arms are beneficiaries of a health program and they could react to interventions differently; in \citep{villar2016indexability}, arms are targets of surveillance who have different locations and probabilities to be exposed; in \citep{GlaMitAns_05_rb_repair}, arms are machines that could require distinct repair schedules.

Although heterogeneity is crucial in the modeling of these applications, most existing work on average-reward WCMDPs establishes asymptotic optimality only for the homogeneous setting where all arms share the same set of model parameters
\citep{WebWei_90,Ver_16_verloop,GasGauYan_23_exponential,GasGauYan_23_whittles,HonXieChe_23,HonXieChe_24,hong2024exponential,Yan_24,HodGla_15,GolAvr_24_wcmdp_multichain}.
Only a few exceptions \citep{Ver_16_verloop,XioWanLi_22,HonXieChe_23} address heterogeneity, but in highly specialized settings.
Among these, a common approach to handle heterogeneity is to consider the \emph{typed heterogeneous} setting, where the $N$ arms are divided into a constant number of types as $N$ scales up, with each type having distinct model parameters.
While heterogeneous WCMDPs have been studied under the finite-horizon total-reward and discounted-reward criteria, these results do not extend to the average-reward setting we consider. 
We review related work in more detail at the end of this section and also in Appendix~\ref{app:additional-related-work}.

The key distinction between the homogeneous (or typed heterogeneous) setting and the fully heterogeneous setting is whether the arms can be divided into a \emph{constant} number of homogeneous groups.
In the former, the system dynamics depends only on the fraction of arms in each state in each homogeneous group.
Thus, the effective dimension of the state space is \emph{polynomial} in $N$.
In contrast, in the fully heterogeneous setting, the state space grows \emph{exponentially} in $N$, making the problem truly high-dimensional.

\paragraph{Our contribution.}
In this paper, we study \emph{fully heterogeneous} WCMDPs.
We propose a policy we call the \emph{ID policy with reassignment}, which generalizes the ID policy in the literature \citep{HonXieChe_24}.
In our policy, we first perform an ID reassignment algorithm to reorder the arms, which ensures proper arm prioritization during policy execution.
We then run a variant of the ID policy adapted to handle heterogeneity, which consists of two phases.
The first phase is a pre-processing phase, where we compute an \emph{optimal single-armed policy} for each arm (denoted as $\pibs_i$ for the $i$-th arm) that prescribes the \emph{ideal action} the arm would take at each state.
The second phase is the real-time phase.
At each time step, the policy iterates over the arms according to their reassigned IDs, and it lets as many arms as possible follow their respective ideal actions while satisfying the budget constraints. 
Unlike the original ID policy, which has only one optimal single-armed policy due to homogeneous arms, our policy computes $N$ optimal single-armed policies, one for each arm.  
Our proposed policy is efficiently computable, with computational complexity polynomial in~$N$.

We prove that the proposed ID policy with reassignment achieves an $O(1/\sqrt{N})$ optimality gap under mild assumptions as the number of arms $N$ becomes large.
Here, the optimality gap refers to the difference between the long-run average reward per arm under our policy and that under the optimal policy.
\emph{This is the first result establishing asymptotic optimality for fully heterogeneous average-reward WCMDPs.}

We remark that the original ID policy was designed for a special case of WCMDPs known as restless bandits, and it is for the homogeneous setting.
While the generalization in our proposed policy is natural, identifying the appropriate generalization and establishing its optimality gap in the fully heterogeneous setting are technically challenging and require new theoretical approaches.

\paragraph{Technical novelty.}
The main technical innovation of the paper is the introduction of a novel Lyapunov function for fully heterogeneous WCMDPs. 
Specifically, to prove the asymptotic optimality of a policy, a key step is to show that the system state is globally attracted to an \emph{optimal region} where most arms can follow the ideal actions generated by their respective optimal single-armed policies $\pibs_i$'s. 
In the homogeneous setting, we can prove such convergence using state aggregation techniques that rely on the symmetry of arms.
However, in the heterogeneous setting, states of different arms cannot be aggregated since arms are no longer symmetric.
Our technique is to \emph{project} arm states onto a set of carefully selected feature vectors, and define the Lyapunov function based on these projections. 
These feature vectors encode the minimal amount of information needed to evaluate the relevant functions of the system state (e.g., instantaneous reward or cost) and predict their future expectations.  
This projection-based Lyapunov function provides a principled way to measure deviations of the system state from the optimal region in a fully heterogeneous setting. 
A more detailed discussion of this approach can be found in \Cref{sec:result-tech-overview}.

Beyond WCMDPs, our techniques have the potential to be applied to more general heterogeneous large stochastic systems.
Heterogeneity has been a topic of strong interest in these systems, but it is known to be a challenging problem with limited theoretical results.
Only recently have there been notable breakthroughs.
\citep{AllGas22_het,AllGas22_graphon} extended the popular mean-field analysis to a class of heterogeneous large stochastic systems for the first time, but the results are only for transient distributions.
Another line of work \citep{ZhaMukWu_24_data_loc_het,ZhaMuk_24_rate_mat_prun} studied heterogeneous load-balancing systems.
They first analyzed the transient distributions and then used interchange-of-limits arguments to extend the results to steady state.
Our method provides a more direct framework for steady-state analysis and has the potential to generalize to a broader range of heterogeneous stochastic systems.

\paragraph{Related work.}
WCMDPs have been extensively studied with a rich body of literature.
Here we briefly overview the most relevant work and refer the reader to Appendix~\ref{app:additional-related-work} for a detailed survey.

We first focus on the \emph{average-reward} criterion.
As mentioned earlier, most existing work considers the \emph{homogeneous setting}.
Early work on WCMDPs primarily focuses on a special case known as the \emph{restless bandit (RB)} problem, where each arm's MDP has a binary action space (active and passive actions) and there is only one budget constraint that limits the total number of active actions across all arms at each time step.
The seminal work by \citet{Whi_88_rb} introduced the RB problem and the celebrated Whittle index policy, which was later shown to achieve an $o(1)$ optimality gap as $N\to\infty$ under a set of conditions \citep{WebWei_90}.
Subsequent work on RBs has focused on designing policies that achieve asymptotic optimality under more relaxed conditions \citep{Ver_16_verloop,HonXieChe_23,HonXieChe_24,Yan_24}, as well as improving the optimality gap to $O(1/\sqrt{N})$ \citep{HonXieChe_23,HonXieChe_24} or $O(\exp(-cN))$ \citep{GasGauYan_23_exponential,GasGauYan_23_whittles,hong2024exponential}.
Among these papers, \citep{Ver_16_verloop,HonXieChe_23} address heterogeneous RBs. However, \citep{Ver_16_verloop} focuses on the \emph{typed heterogeneous} setting, where the $N$ arms are divided into a constant number of types as $N\to\infty$.  
The paper \citep{HonXieChe_23} includes an extension to the fully heterogeneous setting.
However, for their result to yield asymptotic optimality, there need to be further assumptions on the orders of the so-called synchronization times in the paper.
The policies in both papers cannot be straightforwardly extended to general WCMDPs, which have multiple actions, multiple budget constraints, and state-dependent cost functions.

Beyond RBs, work on general average-reward WCMDPs is scarce.
The closest to ours are \citep{HodGla_15, Ver_16_verloop,XioWanLi_22}, which studied WCMDPs with a single budget constraint and established $o(1)$ optimality gaps, but again in the homogeneous setting \citep{HodGla_15} or the \emph{typed heterogeneous} setting \citep{Ver_16_verloop,XioWanLi_22}. 
More recently, \citep{GolAvr_24_wcmdp_multichain} proved the first $o(1)$ optimality gap result for WCMDPs with general budget constraints, but in the \emph{homogeneous} setting.

Under the \emph{finite-horizon total-reward} or \emph{discounted-reward} criteria, there has been more work on heterogeneous settings, including both the typed heterogeneous setting \citep{DaeChoGri_23,GhoNagJaiTam_23_finite_discount} and, more recently, the fully heterogeneous setting \citep{BroSmi_19_rb,BroZha_22,BroZha_23,Zhang24_het}.
However, the optimality gap in these papers generally grows \emph{super-linearly} with the (effective) time horizon, except under restrictive conditions.
Consequently, it is difficult to extend these results to the average-reward setting and still achieve asymptotic optimality.

Although our work focuses on the planning setting where all model parameters are known, there has been growing interest in developing reinforcement learning algorithms for the learning setting with unknown parameters \citep{biswas21index_healthcare,NakGanHsi_21,killian2021q_multi,NakHou_22,XioWanLi_22,killian2022restless,el2023weakly,robledo2022qwi,XiongLi23,AvrBor_22}.
Many of these approaches rely on well-designed planning policies as a foundation to achieve learning efficiency.
In this context, our results can serve as an important building block for developing model-based learning algorithms for fully heterogeneous WCMDPs.

\paragraph{General notation.}
Let $\mathbb{R}$, $\mathbb{N}$, and $\mathbb{N}_+$ denote the sets of real numbers, nonnegative integers, and positive integers, respectively.
Let $[N]\triangleq\{1,2,\dots,N\}$ for any $N\in \mathbb{N}_+$ and $[n_1:n_2]\triangleq\{n_1,n_1+1,\dots,n_2\}$ for $n_1,n_2\in\mathbb{N}_+$ with $n_1\le n_2$.
Let $[0,1]_N=\{i/N\colon i\in\mathbb{N}, 0\le i/N\le 1\}$, the set of integer multiples of $1/N$ in $[0,1]$.
For a matrix ${A}\in\RR^{d\times d}$, we denote its operator norm as $\norm{{A}}_p=\sup_{x\neq \mathbf{0}} \norm{{A}x}_{p} / \norm{x}_p$. 
We use boldface letters to denote matrices, and regular letters to denote vectors and scalars. We write $\R^{\sspa}$ for the set of real-valued vectors indexed by elements of $\sspa$, or equivalently, the set of real-valued functions on $\sspa$; for each $v\in \R^{\sspa}$, let $v(s)$ to denote its element corresponding to $s\in\sspa$.

\section{Problem setup}\label{sec:background}
We consider a weakly-coupled Markov decision process (WCMDP) that consists of $N$ arms.
Each arm has an ID $i\in[N]$ and is associated with a smaller MDP denoted as $\cM_i=\left(\sspa, \aspa,\PP_i, r_i, (c_{k,i})_{k\in[K]}\right)$.
Here $\sspa$ and $\aspa$ are the state space and the action space, respectively, both assumed to be finite; $\PP_i$ describes the transition probabilities with $\PP_i(s'\mid s,a)$ being the transition probability from state $s$ to state $s'$ when action $a$ is taken.
The state transitions of different arms are independent given the actions.
When arm~$i$ is in state $s$ and we take action $a$, a reward $r_i(s,a)$ is generated, as well as $K$ types of costs $c_{k,i}(s,a), k\in[K]$. 
We assume that the costs are nonnegative, i.e., $c_{k,i}(s,a)\ge 0$ for all $i\in\mathbb{N}_+,k\in[K],s\in\sspa,$ and $a\in\aspa$.
Note that we allow the arms to be \emph{fully heterogeneous}, i.e., the $\cM_i$'s can be \emph{all distinct}.

When taking an action for each arm in this $N$-armed system, we are subject to budget constraints.
Specifically, suppose each arm $i$ is in state $s_i$.
Then the actions, $a_i$'s, should satisfy the constraints:
\begin{align}\label{eq:cost}
    \sum_{i\in[N]} c_{k,i}(s_i, a_i)\le \alpha_k N,\quad \forall k\in[K],
\end{align}
where each $\alpha_k>0$ is a constant independent of $N$, and $\alpha_k N$ is referred to as the \emph{budget} for type-$k$ cost.
We assume that there exists an action $0\in\aspa$ that does not incur any type of cost for any arm at any state, i.e., $c_{k,i}(s,0)=0$ for all $k\in[K],i\in[N],s\in\sspa$.
This assumption guarantees that there always exist valid actions (e.g., taking action $0$ for every arm) regardless of the states of the arms.

\paragraph{Policy and system state.}
A policy $\pi$ for the $N$-armed problem specifies the action for each of the $N$ arms, in a possibly history-dependent way.
Under policy $\pi$, let $S_{i,t}^\pi$ denote the state of the $i$th arm at time $t$, and we refer to $\bm{S}^\pi_t \triangleq (S_{i,t}^\pi)_{i\in[N]}$ as the \emph{system state}. Similarly,
let $A_{i,t}^\pi$ denote the action applied to arm $i$ at time $t$, and we refer to $\bm{A}^\pi_t \triangleq (A_{i,t}^\pi)_{i\in[N]}$ as the \emph{system action}. 
In this paper, we also use an alternative representation of the system state, denoted as $\mX_t^\pi$ and defined as follows.
Let $X_{i,t}^\pi=(X_{i,t}^\pi(s))_{s\in\sspa}\in\mathbb{R}^{|\sspa|}$ be a row vector where the entry corresponding to state $s$ is given by $X_{i,t}^{\pi}(s)=\mathbbm{1}\{S_{i,t}^\pi=s\}$; i.e., $X_{i,t}^\pi$ is a one-hot row vector whose $s$'s entry is $1$ if $S_{i,t}^\pi=s$ and is $0$ otherwise.
Then let $\mX_t^\pi$ be an $N\times|\sspa|$ matrix whose $i$th row is $X_{i,t}^\pi$.
It is easy to see that $\mX_t^\pi$ contains the same information as $\bm{S}^\pi_t$, and we refer to both of them as the system state.
In this paper, we often encounter vectors like $X_{i,t}^\pi=(X_{i,t}^\pi(s))_{s\in\sspa}$, whose entries correspond to different states in $\sspa$.
For such vectors, say $u$ and $v$, we use the inner product to write a sum for convenience $\innerproduct{u}{v}\triangleq \sum_{s\in\sspa}u(s)v(s)$.
We sometimes omit the superscript $\pi$ when it is clear from context.

\paragraph{Maximizing average reward.}
Our objective is to maximize the long-run time-average reward subject to the budget constraints.
To be more precise, we follow the treatment for maximizing average reward in \citep{Put_05}. 
For any policy $\pi$ and an initial state $\bm{S}_0$ of the $N$-armed system, 
consider the \emph{limsup} average reward $R^+(\pi,\bm{S}_0)$ and the \emph{liminf} average  $R^-(\pi,\bm{S}_0)$, defined as $R^+(\pi,\bm{S}_0)=\lim\sup_{T\to\infty}\frac{1}{T}\sum_{t=0}^{T-1}\frac{1}{N}\sum_{i\in[N]}\E{r_i(S_{i,t}^\pi, A_{i,t}^\pi)}$ and $R^-(\pi,\bm{S}_0)=\lim\inf_{T\to\infty}\frac{1}{T}\sum_{t=0}^{T-1}\frac{1}{N}\sum_{i\in[N]}\E{r_i(S_{i,t}^\pi, A_{i,t}^\pi)}$. 
If $R^+(\pi,\bm{S}_0)=R^-(\pi,\bm{S}_0)$, then the average reward of policy $\pi$ under the initial condition $\bm{S}_0$ exists and is defined as
\begin{equation}
    R(\pi,\bm{S}_0) = R^+(\pi,\bm{S}_0) = R^-(\pi,\bm{S}_0)=\lim_{T\to\infty}\frac{1}{T}\sum_{t=0}^{T-1}\frac{1}{N}\sum_{i\in[N]}\E{r_i(S_{i,t}^\pi, A_{i,t}^\pi)}.
\end{equation}
Note that these reward notions divide the total reward from all arms by the number of arms, $N$, measuring the reward \emph{per arm}.
The WCMDP problem is to solve the following optimization problem:
\begin{subequations}\label{eq:WCMDP}
\begin{align}
    \label{eq:N-arm-formulation} 
    \underset{\text{policy } \pi}{\text{maximize}} &\quad R^-(\pi,\bm{S}_0)\\
    \text{subject to}  
    &\mspace{12mu}  \sum_{i\in[N]} c_{k,i}(S_{i,t}^\pi, A_{i,t}^\pi)\le \alpha_k N,\quad \forall k\in[K],\forall t\ge 0. \label{eq:hard-budget-constraint}
\end{align}
\end{subequations}
Let the optimal value of this problem be denoted as $R^*(N,\bm{S}_0)$.
Note that since the WCMDP is an MDP with finite state and action space, if we replace the $R^-(\pi,\bm{S}_0)$ in the objective \eqref{eq:N-arm-formulation} with $R^+(\pi,\bm{S}_0)$, the optimal value stays the same \citep[Proposition 9.1.6]{Put_05}. 

\paragraph{Asymptotic optimality.}
Recall that exactly solving the WCMDP problem is PSPACE-hard \citep{PapTsi_99_pspace}. 
In this paper, our goal is to design a policy $\pi$ that is \emph{efficiently computable} and \emph{asymptotically optimal} as $N\to\infty$, with the following notion for asymptotic optimality.
For any policy $\pi$, we define its \emph{optimality gap} as $R^*(N,\bm{S}_0)-R^-(\pi,\bm{S}_0)$.
We say the policy $\pi$ is \emph{asymptotically optimal} if 
\begin{equation}
    R^*(N,\bm{S}_0)-R^-(\pi,\bm{S}_0)=o(1) \quad \text{as} \quad N\to\infty.
\end{equation}
When we take this asymptotic regime as $N\to\infty$, we keep the number of constraints, $K$, as well as the budget coefficients, $\alpha_1,\alpha_2,\dots,\alpha_K$, fixed.
We assume that the reward functions and cost functions are uniformly bounded, i.e., $\sup_{i\in\mathbb{N}_+}\max_{s\in\sspa,a\in\aspa}|r_i(s,a)|\triangleq r_{\max}<\infty$
and
$\sup_{i\in\mathbb{N}_+}\max_{k\in[K],s\in\sspa,a\in\aspa}c_{k,i}(s,a)\triangleq c_{\max}<\infty$.
This notion for asymptotic optimality is consistent with that in the existing literature (e.g., \citep[Definition 4.11]{Ver_16_verloop}). We are interested in not only achieving asymptotic optimality but also characterizing the \emph{order} of the optimality gap.

In the remainder of this paper, we focus on stationary Markov policies, which are sufficient for achieving the optimal value because the WCMDP problem is an MDP with finite state and action spaces \citep[][Theorem 9.1.8]{Put_05}. 
Under any stationary Markovian policy, the long-run reward $R(\pi,\bm{S}_0)= R^+(\pi,\bm{S}_0) = R^-(\pi,\bm{S}_0)$ is always well-defined \citep[][Theorem 8.1.1]{Put_05}.

\paragraph{LP relaxation and an upper bound on optimality gap.}
We consider the linear program (LP) below, which will play a critical role in performance analysis and policy design:
\begin{subequations}\label{eq:lp}
\begin{align}
    R^{\rel}_N\triangleq&\underset{(y_i(s,a))_{i\in [N], s\in\sspa,a\in\aspa}}{\text{maximize}}\mspace{9mu}\frac{1}{N}\sum_{i\in[N]} \sum_{s\in\sspa,a\in\aspa} y_i(s,a)r_i(s,a)\\
    \label{eq:lp:budget-constraint}
    &\mspace{32mu}\text{subject to}\mspace{42mu}\frac{1}{N}\sum_{i\in[N]}\sum_{s\in\sspa,a\in\aspa} y_i(s,a)c_{k,i}(s,a)\leq \alpha_k,\;\forall k\in[K],\\
    \label{eq:lp:stationarity-constraint}
    &\mspace{130mu} \sum_{s'\in\sspa,a'\in\aspa} \PP_i(s\mid s',a')y_i(s',a')=\sum_{a\in\aspa}y_i(s,a),\;\forall s\in\sspa,\forall i\in[N],\\
    \label{eq:lp:probability-constraint}
   &\mspace{135mu}\sum_{s'\in\sspa,a'\in\aspa}y_i(s',a')=1, \; y_i(s,a)\geq 0, \; \forall s\in\sspa,  \forall a\in \aspa,\forall i\in[N].
\end{align}
\end{subequations}
Lemma~\ref{lem:lp upper bound} below establishes a connection between this LP and the WCMDP.
\begin{restatable}{lemma}{lprelaxation}\label{lem:lp upper bound}
The optimal value of any $N$-armed WCMDP problem is upper bounded by the optimal value of the corresponding linear program in \eqref{eq:lp}, i.e.,
    \begin{align*}
        R^*(N,\bm{S}_0)\le R_N^{\rel},\quad\forall N,\forall \bm{S}_0.
    \end{align*}
\end{restatable}

An immediate implication of Lemma~\ref{lem:lp upper bound} is that for any policy $\pi$, its optimality gap is upper bounded as
\begin{equation}
    R^*(N,\bm{S}_0)-R^-(\pi,\bm{S}_0) \le R_N^{\rel}-R^-(\pi,\bm{S}_0).
\end{equation}
Therefore, to derive an upper bound for the optimality gap, it suffices to control $R_N^{\rel}-R^-(\pi,\bm{S}_0)$, which we will show is $O(1/\sqrt{N})$ in \Cref{thm:opt-gap-bound}. 

To see the intuition of Lemma~\ref{lem:lp upper bound}, we interpret the optimization variable $y_i(s,a)$ as the long-run fraction of time arm~$i$ spends in state $s$ and takes action $a$.
We refer to $y_i(s,a)$ as arm~$i$'s \emph{state-action frequency} for the state-action pair $(s,a)$.
Then the constraints in \eqref{eq:lp:budget-constraint} of the LP can be viewed as relaxations of the budget constraints in \eqref{eq:hard-budget-constraint} for the WCMDP.
The constraints in \eqref{eq:lp:stationarity-constraint}--\eqref{eq:lp:probability-constraint} guarantee that $y_i(s,a)$'s are proper stationary time fractions.
Therefore, the LP is a relaxation of the WCMDP and thus achieves a higher optimal value.
The formal proof of Lemma~\ref{lem:lp upper bound} is provided in Appendix~\ref{app:proof-lem-lp-upper-bound}.

Our LP \eqref{eq:lp} serves a similar role to the LP used in previous work on restless bandits and WCMDPs (see, e.g., \citep{WebWei_90,Ver_16_verloop,GasGauYan_23_exponential,HodGla_15,XioWanLi_22,GolAvr_24_wcmdp_multichain,HonXieChe_23}), but with different forms and dependencies on $N$. 
Both our LP and the LP in previous work relax the hard budget constraints to time-average constraints. 
However, in the homogeneous arm setting \citep{WebWei_90,GasGauYan_23_exponential,HodGla_15,GolAvr_24_wcmdp_multichain}, the LP has only one set of state-action frequencies $y(s,a)$, and the LP is independent of $N$.
As a result, both the optimal value of the LP and the complexity of solving it are independent of $N$. 
Prior work for the typed-heterogeneous setting \citep{XioWanLi_22,Ver_16_verloop,HonXieChe_23} divides the arms into a constant number $K$ of types and defines a set of state-action frequency $y_k(s,a)$ for each type $k$. The optimal value of the resulting LP and the complexity of solving it depend on the number of types $K$.
\citet{HonXieChe_23} includes a generalization to the fully heterogeneous setting, resulting in an LP similar to ours, although restricted to restless bandits.

In our LP, we define a separate set of state-action frequencies $y_i(s,a)$ for each arm $i\in[N]$, making the LP explicitly depend on $N$. 
Therefore, the optimal value $R_N^{\rel}$ depends on $N$, and the complexity of solving the LP grows with $N$. 
Nevertheless, because the number of variables and constraints scales linearly with $N$, our LP can still be solved in polynomial time.

\section{ID policy with reassignment}
In this section, we introduce the ID policy with reassignment, generalized from the ID policy designed for homogeneous restless bandits in the literature \citep{HonXieChe_24}.
Our policy first performs an ID reassignment procedure, and then proceeds to run a variant of the ID policy adapted to handle heterogeneity.
We begin by introducing a building block of our policy, referred to as optimal single-armed policies, followed by the ID reassignment algorithm and the execution of the adapted ID policy.

\paragraph{Optimal single-armed policies.}
Once we obtain a solution to the LP in \eqref{eq:lp}, we can construct a policy for each arm~$i$, which we refer to as an \emph{optimal single-armed policy} for arm~$i$.
In particular, let $(y^*_i(s,a))_{i\in [N], s\in\sspa,a\in\aspa}$ be an arbitrary optimal solution to the LP in \eqref{eq:lp}. 
Then for arm~$i$, the optimal single-armed policy, $\pibar^*_i$, is defined as
\begin{equation}\label{eq:optimal-single-armed-policy}
    \bar{\pi}_i^*(a\mid s) = \begin{cases}
        \frac{y^*_i(s,a)}{\sum_{a\in\aspa}y^*_i(s,a)}, & \text{if}~ \sum_{a\in\aspa}y^*_i(s,a) >0,\\
        \frac{1}{\cardA}, &  \text{if} ~\sum_{a\in\aspa}y^*_i(s,a)=0,
    \end{cases}
\end{equation}
where $\bar{\pi}_i^*(a\mid s)$ is the probability of taking action $a$ given that the arm's current state is $s$.
Note that due to heterogeneity, this optimal single-armed policy $\pibar^*_i$ can be different for different arms.

The rationale behind these policies is as follows.
If each arm $i$ individually follows its optimal single-armed policy $\pibar_i^*$, then the average reward per arm (total reward divided by $N$) achieves the upper bound $R_N^{\rel}$ given by the LP.
However, this strategy only guarantees that the budget constraints are satisfied in a \emph{time-average} sense, rather than conforming to the \emph{hard} constraints in the original $N$-armed WCMDP.
Thus, having each arm follow its optimal single-armed policy is not a valid policy for the original $N$-armed problem.
Nevertheless, these optimal single-armed policies $\pibar_i^*$'s serve as a guide for how the arms should ideally behave to maximize rewards.
The ID policy uses the $\pibar_i^*$'s as a reference.
It is then designed to ensure that even under the hard budget constraints, most arms follow their optimal single-armed policies most of the time, yielding a diminishing gap to $R_N^{\rel}$ in reward.

\begin{algorithm}[t]
     \caption{ID reassignment}\label{alg:id-assign}
     \begin{algorithmic}[1]
        \STATE \textbf{Input:} optimal state-action frequencies $(y^*_i(s,a))_{i\in [N], s\in\sspa,a\in\aspa}$, budgets $(\alpha_k)_{k\in[K]}$
        \STATE \textbf{Output:} new arm ID, recorded in $\newID(i)$, for each arm with old ID $i\in[N]$
        \STATE Compute $(C_{k,i}^*)_{i\in[N],k\in[K]}$ and the set of active constraints $\cA$ using \eqref{eq:active}
        \IF{$\cA=\emptyset$}
        \STATE $\newID(i)\gets i$ for all $i\in[N]$  \hfill \emph{$\triangleright$ No need for ID reassignment}
        \ELSE 
        \STATE Initialize $\cF \gets \emptyset$ \hfill\emph{$\triangleright$ Set of arms that have been assigned new IDs}
        \STATE Initialize $\cD_k\gets\{i\in[N]\colon C_{k,i}^*\ge \costIDreassign\}$ for all $k\in\cA$\\
        \STATE $\costIDreassign\gets \alpha_{\min}/4\triangleq\min_{k\in[K]}\alpha_k/4$;\;
        $\gsizeIDreassign \gets \left\lceil \frac{(c_{\max}-\costIDreassign)K}{\alpha_{\min}/2-\costIDreassign}\right\rceil$
        \FOR{$\ell =0,1,\dots,\lfloor N/\gsizeIDreassign\rfloor-1$}
        \STATE $\cI(\ell) \gets [\ell\gsizeIDreassign+1:(\ell+1)\gsizeIDreassign]$;\; $j\gets \ell\gsizeIDreassign+1$
        \FOR{$k\in\cA$}\label{alg:line:reassign-start}
            \IF{$\sum_{i\in \cF}C_{k,i}^*\mathbbm{1}\{\newID(i)\in \cI(\ell)\}<\costIDreassign$}
            \STATE Pick any $i$ from $\cD_k$ and set $\newID(i)\gets j$; remove $i$ from $\cD_{k'}$ for all $k'$; add $i$ to $\cF$
            \STATE $j\gets j+1$
            \ENDIF
        \ENDFOR\label{alg:line:reassign-end}
        \ENDFOR
        \STATE For all $i\in [N]\setminus \cF$, assign values to their $\newID(i)$'s randomly from $[N]\setminus \{\newID(i')\colon i'\in \cF\}$
        \ENDIF
     \end{algorithmic}
\end{algorithm}
\paragraph{ID reassignment.}
We first define a few quantities that will be used in the ID reassignment algorithm.
For each arm $i\in[N]$ and each cost type $k\in[K]$, the expected cost under the optimal single-armed policy is defined as $C_{k,i}^* =\sum_{s\in\sspa,a\in\aspa}y_i^*(s,a)c_{k,i}(s,a)$.
Based on $C_{k,i}^*$'s, we divide the budget constraints into \emph{active} constraints and \emph{inactive} constraints as follows.
For each cost type $k\in[K]$, we say the type-$k$ budget constraint is \emph{active} if
\begin{equation}\label{eq:active}
    \sum_{i\in[N]}C_{k,i}^*\ge \frac{\alpha_k}{2}N,
\end{equation}
and \emph{inactive} otherwise.
Let $\cA\subseteq [K]$ denote the set of cost types corresponding to active constraints.
Note that replacing $\alpha_k N/2$ with any constant and strict fraction of $\alpha_k N$ will not change the results.

Based on the costs $C_{k,i}^*$'s and the active constraints, the ID reassignment algorithm rearranges arms so that the cost incurred by each contiguous segment of arms is approximately proportional to the length of the segment.
We give a brief explanation of how the reassigned IDs affect the execution of the policy at the end of this section.
The algorithm is formally described in Algorithm~\ref{alg:id-assign}, with more details and properties provided in Appendix~\ref{app:more-ID-reassignment}.
In the remainder of this paper, we use the \emph{reassigned IDs} to refer to arms, i.e., arm $i$ refers to the arm whose new ID assigned by the ID reassignment algorithm is~$i$.

\begin{algorithm}[t]
    \caption{ID policy with reassignment}
    \label{alg:id policy}
    \begin{algorithmic}[1]
        \STATE
        \textbf{Input:} $N$-armed WCMDP instance $(\cM_i)_{i\in[N]}$
        \STATE \textbf{Preprocessing:}        
        \STATE \hspace{0.7em} Solve the LP in \eqref{eq:lp} and obtain the optimal state-action frequencies $(y^*_i(s,a))_{i\in [N], s\in\sspa,a\in\aspa}$
        \STATE \hspace{0.7em} Calculate the optimal single-armed policies $(\pibar_i^*)_{i\in[N]}$ using \eqref{eq:optimal-single-armed-policy} 
        \STATE \hspace{0.7em} Perform ID reassignment using Algorithm~\ref{alg:id-assign}
        \STATE \textbf{Real-time:}
        \FOR{$t=0,1,2,\cdots$} 
            \STATE Sample ideal actions $\widehat{A}_{i,t}\sim \pibar_i^*(\cdot\mid S_{i,t})$ for all $i\in[N]$
            \STATE $I\gets 1$ 
            \WHILE{$\sum_{i\in[I]}c_{k,i}(S_{i,t},\widehat{A}_{i,t})\leq \alpha_k N, \forall k\in[K]$}
                \STATE For arm $I$, take action $A_{I,t}=\widehat{A}_{I,t}$; $\quad I\gets I+1$
            \ENDWHILE
        \STATE For each arm $i\in\{I,I+1,\dots,N\}$, take action $A_{i,t}=0$
        \ENDFOR
    \end{algorithmic}
\end{algorithm}
\paragraph{Constructing ID policy.}
We are now ready to describe our generalized ID policy, formally described in Algorithm~\ref{alg:id policy}.
The policy begins with a one-time preprocessing phase: we solve the associated LP, construct the optimal single-armed policies, and reassign arm IDs using the ID reassignment algorithm (Algorithm~\ref{alg:id-assign}).
After the preprocessing, the policy proceeds at each time step $t$ as follows.
For each arm~$i$ (where $i$ is the reassigned ID), we first sample an action $\widehat{A}_{i,t}$, referred to as an \emph{ideal action}, from the optimal single-armed policy $\pibar_i^*(\cdot \mid S_{i,t})$.
We then attempt to execute these ideal actions, i.e., set the real actions equal to the ideal actions, in ascending order of arm IDs, starting from $i=1$, then $i=2$, and so on.
We continue the attempt until we have used up at least one type of cost budget, at which point we let the remaining arms take action $0$ (the no-cost action). 

This ID policy is a natural generalization of the ID policy designed for homogeneous restless bandits \cite{HonXieChe_24}.
At a high level, using arm IDs to decide the priority order for executing ideal actions guarantees that a subset of arms (those with smaller IDs) can persistently follow their ideal actions.
This persistency gives these arms time to converge to their optimal state-action frequencies, which in turn allows their instantaneous costs to converge to steady-state values. This convergence creates slack in the budget constraints, thereby allowing more arms to follow their ideal actions. 
In contrast, if we do not use IDs but instead randomly select a subset of arms to follow their ideal actions, the convergence may be disrupted. 
Indeed, \cite{HonXieChe_23} provides an example where this randomized strategy fails to achieve asymptotic optimality in restless bandits.

When there is no ambiguity, we refer to the ID policy with reassignment simply as the ID policy.

\paragraph{How the reassigned IDs affect the policy execution.}
In the heterogeneous setting, arms differ in their cost consumptions under their optimal single-armed policies.
The ID reassignment algorithm is designed to prevent ``plateaus'' in cumulative cost as we progress from smaller to larger IDs.
If such a plateau exists, the subset of arms allowed to follow their ideal actions (determined by the budget constraints) can become sensitive to the randomness in action sampling, potentially leading to performance instability.
The ID reassignment algorithm ensures a regularity property of the cost consumptions, stated in Lemma~\ref{lem:positiveC} in Appendix~\ref{app:more-ID-reassignment}, which eliminates such plateaus.

\section{Main results and technical overview}\label{sec:result-tech-overview}
Before we present the main results, we first state our main assumption.
This assumption is for the optimal single-armed policies $\pibar^*_i$'s.
Note that each $\pibar^*_i$ is a stationary Markov policy.
Therefore, under this policy, the state of arm~$i$ forms a Markov chain.
Let the transition probability matrix of this Markov chain be denoted as $P_i=(P_i(s,s'))_{s\in\sspa,s'\in\sspa}$, where the row index is the current state $s$ and the column index is the next state $s'$.
Then $P_i(s,s')$ can be written as
\begin{equation}
    P_i(s,s') = \sum_{a\in \aspa} \PP_i(s'\mid s,a)\bar{\pi}_i^*(a\mid s).
\end{equation}
One can verify that the stationary distribution of this Markov chain is $\mu_i^*=(\mu_i^*(s))_{s\in\sspa}$ with $\mu_i^*(s)=\sum_{a\in\aspa}y^*_i(s,a)$, which we refer to as the \emph{optimal state distribution} for arm~$i$. 
Let $\tmix_i$ be the \emph{mixing time} of this Markov chain, defined as
\begin{equation}
    \tmix_i=\max_{s\in\sspa} \min\left\{t\in\mathbb{N}\colon \norm{P_i^t(s,\cdot) - \mu^*_i(\cdot)}_1 \leq 1/e \right\},
\end{equation}
where $P_i^t$ is the $t$-step transition probability matrix. The mixing time $\tmix_i$ is finite if the Markov chain $P_i$ is unichain (one recurrent class, possibly with transient states) and aperiodic.

\begin{assumption}\label{ass:unichain}
    For each arm $i\in\mathbb{N}_+$, the induced Markov chain under the optimal single-armed policy $\bar{\pi}_i^*$ is an aperiodic unichain. 
    Furthermore, the mixing times of these Markov chains have a uniform upper bound; i.e., there exists a positive $\tau$ such that for all $i\in\mathbb{N}_+$, 
    \begin{equation}
        \tmix_i \leq \tau.
    \end{equation}
\end{assumption}

We remark that in the homogeneous or typed heterogeneous settings, once we make the aperiodic unichain assumption in Assumption~\ref{ass:unichain}, the uniform upper bound on mixing times automatically exists.

Next, we state our main theorem, \Cref{thm:opt-gap-bound}, whose proof is provided in \Cref{sec:proof-main}.
\begin{theorem}
\label{thm:opt-gap-bound}
    Consider an $N$-armed WCMDP problem satisfying Assumption~\ref{ass:unichain}, with initial system state $\bm{S}_0$. Let policy $\pi$ be the ID policy with reassignment (Algorithm~\ref{alg:id policy}). Then the optimality gap of $\pi$ is bounded as
    \begin{align*}
        R^*(N,\bm{S}_0)-R(\pi,\bm{S}_0) \leq\frac{C_{\ID}}{\sqrt{N}},
    \end{align*}
    where $C_{\ID}$ is a positive constant independent of $N$. 
\end{theorem}

We re-emphasize that our proposed ID policy with reassignment is the first efficiently computable policy that achieves an $O(1/\sqrt{N})$ optimality gap for fully heterogeneous average-reward WCMDPs.
In contrast, the best-known optimality gap for efficiently computable policies for average-reward WCMDPs is $o(1)$, achieved only under restrictive budget constraints and typed-heterogeneity.

We comment that the primary goal of this paper is to characterize the optimality gap in terms of its order in $N$, which is in line with the main focus of the large body of prior work on restless bandits and WCMDPs.
While our analysis also gives an explicit expression for the constant $C_{\ID}$, which shows that $C_{\ID}=O(K^5 \max\{r_{\max}, c_{\max}\}^7 \tau^4/\alpha_{\min}^6)$, we have not attempted to optimize its dependence on other problem parameters, either through refined analysis or alternative policy design.

\begin{remark}[Generalization of result]
    Our result can be generalized to the setting where a $g(N)$ fraction of arms have unbounded mixing times, and the mixing times of the remaining arms scale with $N$ with an upper bound $\tau$. In this case, we can modify the ID policy by reassigning this $g(N)$ fraction of arms the largest IDs, effectively ignoring these arms. Applying \Cref{thm:opt-gap-bound} to the remaining arms then implies $R^*(N,\bm{S}_0)-R(\pi,\bm{S}_0) \leq O\big(K^5 \max\{r_{\max}, c_{\max}\}^7 \tau^4/ (\alpha_{\min}^6 \sqrt{N})  + r_{\max} g(N)\big)$. Consequently, this modified policy is asymptotically optimal when $g(N) = o(1)$ and $\tau = o(N^{1/8})$. 
\end{remark}

Adapting the proof of \Cref{thm:opt-gap-bound} also gives the following finite-time bound (see \Cref{app:finite-time-bound-pf}). 

\begin{proposition}[Finite-time bound]
    \label{prop:finite-time-bound}
    Under the same conditions as \Cref{thm:opt-gap-bound}, we have for any $T\ge 1$,
    \begin{equation}
        R^*(N, \bm{S}_0) - \frac{1}{TN}\sum_{t=0}^{T-1}\sum_{i\in[N]}\E{r_i(S_{i,t}^\pi, A_{i,t}^\pi)} \leq \frac{C_\ID}{\sqrt{N}} + \frac{C_\finite}{T},
    \end{equation}
    where $C_\finite$ is another positive constant independent of $N$. 
\end{proposition}

\subsection*{Technical overview}

Our technical approach uses the Lyapunov drift method, which has found widespread applications in queueing systems, Markov decision processes, reinforcement learning, and so on.
While the basic framework of the drift method is standard, the \emph{key challenge} lies in constructing the right Lyapunov function with the desired properties, where the difficulty is exacerbated by the full heterogeneity of the problem under study.
Our construction of such a Lyapunov function is highly novel, yet still natural.
We reiterate that fully heterogeneous, high-dimensional stochastic systems are poorly understood in the existing literature.
Our approach opens up the possibility of analyzing the steady-state behavior of such systems through the Lyapunov drift method.

In the remainder of this section, we consider the ID policy, also referred to as the policy $\pi$.
Let $\mX_t$ denote the system state under it, with the superscript $\pi$ omitted for brevity.
To make this overview more intuitive, here let us assume that $\mX_t$ converges to its steady state $\mX_{\infty}$ in a proper sense such that taking expectations in steady state is the same as taking time averages.
However, note that our formal results do not need this assumption and directly work with time averages.
We call a function $V$ a Lyapunov/potential function if it maps each possible system state to a nonnegative real number.

\paragraph{General framework of the drift method.}
Here we briefly describe the general framework of the drift method when applied to our problem.
The goal is to construct a Lyapunov function $V$ such that
\begin{enumerate}[label=(\textbf{C\arabic*}),leftmargin=4em]
    \item\label{cond:upper} $R_N^{\rel}-R(\pi,\bm{S}_0) \le C_1\E{V(\mX_{\infty})}/N+O(1/\sqrt{N})$ for some constant $C_1$;
    \item\label{cond:drift} (Drift condition) $\EE\mbracket{V(\mX_{t+1})\mid \mX_t}-V(\mX_t)\leq -C_2V(\mX_t)+ O(\sqrt{N})$ for a constant $C_2$.
\end{enumerate}
The drift condition requires that on average, the value of $V$ approximately decreases (ignoring the additive $O(\sqrt{N})$) after a time step.
The drift condition implies a bound on $\E{V(\mX_{\infty})}$.
To see this, let $\mX_t$ follow the steady-state distribution, which means $\mX_{t+1}$ also follows the steady-state distribution, and take expectations on both sides of the inequality.
Then we get $0=\EE\mbracket{V(\mX_{t+1})}-\E{V(\mX_t)}\leq -C_2\E{V(\mX_t)}+ O(\sqrt{N})$, which implies $\E{V(\mX_{\infty})}=\E{V(\mX_t)}=O(\sqrt{N}).$
Combining this with \ref{cond:upper} proves the desired $O(1/\sqrt{N})$ upper bound on the optimality gap.

\paragraph{Key challenge: constructing Lyapunov function.}
We highlight this challenge by contrasting the homogeneous setting and the heterogeneous setting.
In the \emph{homogeneous} setting, there is only one optimal state distribution, $\mu^*$.
The Lyapunov function in \citep{HonXieChe_24} is defined based on the distance between the \emph{empirical state distribution} across arms and $\mu^*$.  Specifically, it is based on a set of functions $(h(\mX_t,D))_{D\subseteq [N]}$ defined as:
\begin{equation}
    h(\mX_t,D)=\norm{\mX_t(D)-m(D)\mu^*},
\end{equation}
where $\mX_t(D)=(X_t(D,s))_{s\in\sspa}$ denotes within $D$, the number of arms in each state $s$, divided by $N$; $m(D)=|D|/N$; and the norm $\norm{\cdot}$ is a properly defined norm.
The idea is that if all arms in $D$ follow the optimal single-armed policy, the state distribution of each arm in $D$ gets closer to $\mu^*$, and thus $\mX_t(D)$ gets closer to $m(D)\mu^*$ over time.

In the \emph{heterogeneous} setting, we also want to construct a Lyapunov function $h(\mX_t, D)$ to witness the convergence of any set of arms $D$ if they  follow the optimal single-armed policies. 
However, unlike the homogeneous setting, now it no longer makes sense to aggregate arm states into an empirical state distribution, since each arm's dynamics is distinct.
Instead, our Lyapunov function considers $X_{i,t}-\mu_i^*$, where recall $X_{i,t}(s)$ is the indicator that arm~$i$'s state is $s$ at time $t$. 
A naive first attempt is to construct the Lyapunov function from the pointwise distances, $\norm{X_{i,t}-\mu_i^*}$ for each arm~$i$, with a properly defined norm $\norm{\cdot}$. 
However, the pointwise distances are very noisy: $\norm{X_{i,t}-\mu_i^*}$  could be large even when the state of arm~$i$ independently follows the distribution $\mu_i^*$ for each $i$, a situation when we should view the set of arms as already converged. 

Intuitively, to make the Lyapunov function properly reflect the convergence of the set of arms (referred to as ``the system'' in the rest of the section) following the optimal single-armed policies, we would like it to depend less strongly on the state of each individual arm and focus more on the collective properties of the whole system. 
Our idea is to \emph{project} the system state onto a properly selected set of \emph{feature vectors}, and construct the Lyapunov function based on how far these projections are from the projections of the optimal state distributions $(\mu^*_i)_{i\in[N]}$. 
Then what features of the system state do we need to determine whether it has converged in a proper sense? 
The first feature we consider is the instantaneous reward of the system, $\sum_{i\in D} \innerproduct{X_{i,t}}{r_i^*}$, where $r_i^*\in \R^{\sspa}$ is the reward function of arm $i$ under $\pibs_i$, and recall that the inner product is defined between two vectors whose entries correspond to states in $\sspa$. 
We also want to keep track of the \emph{$\ell$-step ahead expected reward}, $\sum_{i\in D} \innerproduct{X_{i,t}P_i^{\ell}}{r_i^*}$, for each $\ell\in \mathbb{N}_+$. 
Intuitively, if $\sum_{i\in D} \innerproduct{(X_{i,t} -\mu_i^*)P_i^{\ell}}{r_i^*}$ is small for each $\ell\in \mathbb{N}$, the reward of the system should remain close to that under the optimal state distributions $(\mu^*_i)_{i\in[N]}$ for a long time; conversely, if the state of each arm~$i$ independently follows $\mu^*_i$, each of these features should be small as well. 
We also consider the \emph{$\ell$-step ahead expected type-$k$ cost} for each $\ell\in \mathbb{N}$ and $k\in [K]$ as features, defined analogously.  

Combining the above ideas, for any set of arms $D$, we let the Lyapunov function $h(\mX_t, D)$ be the supremum of the differences between $\mX_t$ and $\mu^*$ in all the features directions defined above, under proper weightings: 
\begin{equation}\label{eq:h-def-overview}
    h(\mX_t,D) = \max_{g\in \mathcal{G}} \sup_{\ell\in\mathbb{N}} \abs{\sum_{i\in D} \innerproduct{(X_{i,t}-\mu_i^*) P_i^{\ell} / \gamma^{\ell}}{g_i}},
\end{equation}
where $\gamma = \exp(-1/(2\tau))$ for $\tau$ defined in \Cref{ass:unichain}; each element $g\in \mathcal{G}$ is either $g = (r_i^*)_{i\in[N]}$, or corresponds to the type-$k$ cost for some $k\in[K]$ (See \Cref{sec:proof-main} for the definition of $\mathcal{G}$). 
Note that \emph{dividing} each term by powers of $\gamma$ is another interesting trick, which induces a negative drift in $h(\mX_t, D)$ under the optimal single-armed policies, as demonstrated in the proof of \Cref{lem:drift}. 

Now with the set of functions $(h(\mX_t,D))_{D\subseteq [N]}$ defined, we generalize the idea of focus sets in \citep{HonXieChe_24} to convert $(h(\mX_t,D))_{D\subseteq [N]}$ into a Lyapunov function $V(\mX_t)$.
We prove that $V$ satisfies \ref{cond:upper} and \ref{cond:drift} using the structure of $(h(\mX_t,D))_{D\subseteq [N]}$.

\begin{remark}
The idea for constructing $h(\mX_t, D)$ is potentially useful for analyzing other heterogeneous stochastic systems. 
At a high level, projecting the system state onto a set of feature vectors (and their future expectations) can be roughly viewed as aggregating system states whose relevant performance metrics \emph{remain close for a sufficiently long time}. 
This idea provides a new way to measure the distance between two system states in a heterogeneous system, and this distance notation enjoys similar properties as that in a homogeneous system, without resorting to symmetry.  
\end{remark}

\section{Experiments}
\label{sec:experiments}

\begin{wrapfigure}{r}{0.38\textwidth}
\vspace{-12pt}
    \centering
    \includegraphics[width=0.95\linewidth]{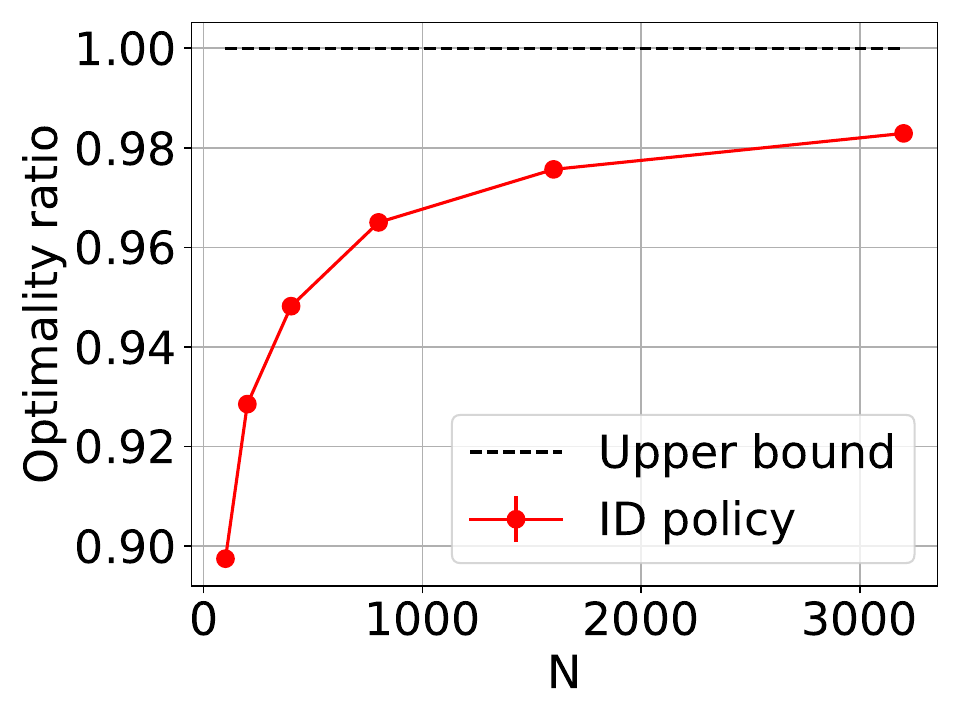}
    \vspace{-6pt}
    \caption{Asymptotic optimality of ID policy.}
    \label{fig:multi-constr}
   \vspace{8pt}
    
    \includegraphics[width=0.95\linewidth]{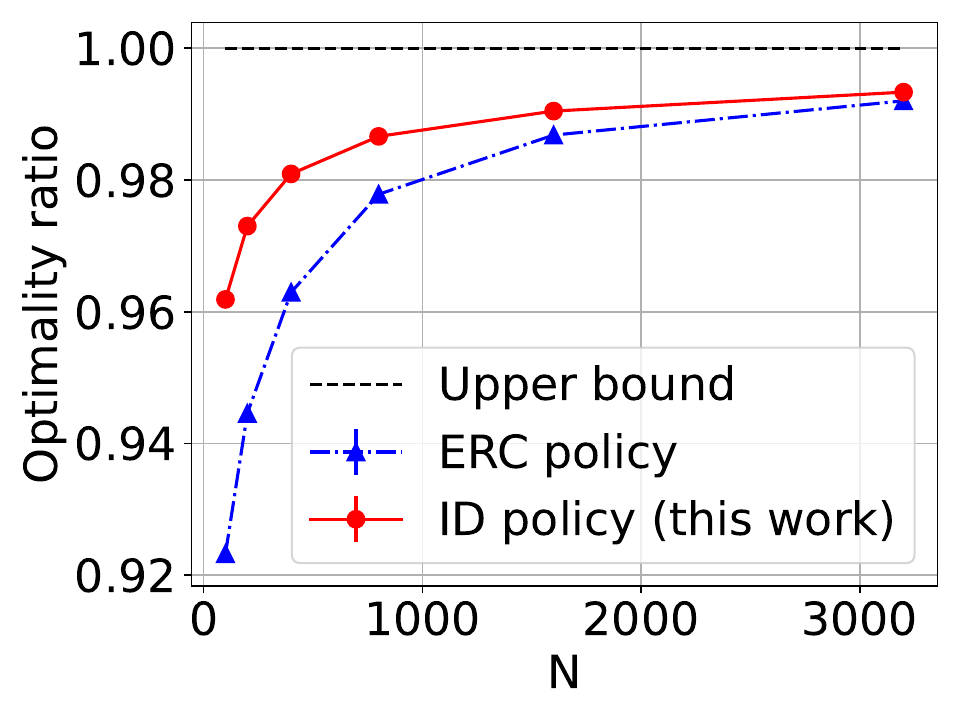}
    \vspace{-6pt}
    \caption{ID policy vs ERC policy.}
    \label{fig:single-constr}
    \vspace{-6pt}
\end{wrapfigure}

In this section, we perform two sets of experiments to illustrate the numerical performance of the proposed ID policy for fully heterogeneous WCMDPs.

In the first set of experiments, we demonstrate the asymptotic optimality of the ID policy.
We increase the number of arms as $N \in \{100, 200, 400, 800, 1600, 3200\}$.
Each arm's MDP has $10$ states and $4$ actions, with parameters generated uniformly at random in a proper sense.
The $N$-armed problem has $4$ budget constraints, with cost functions also generated randomly.
More details are provided in \Cref{app:exp_details}.
We simulate the policy for $2 \times 10^4$ time steps over $4$ replications for each $N$.
To illustrate the performance more clearly, we measure the \emph{optimality ratio}, defined as the ratio between the long-run average reward achieved by a policy and the LP relaxation upper bound $R^{\text{rel}}_N$.
Confidence intervals are calculated using the batch means method with a batch size of $4000$, but they are typically too small to be visible on figures.
Figure~\ref{fig:multi-constr} shows that the optimality ratio of the ID policy becomes increasingly close to $1$ as $N$ increases.

In the second set of experiments, we compare the ID policy with the ERC policy proposed in \citep{XioWanLi_22}.
As discussed in \Cref{sec:introduction}, only a few prior papers address heterogeneous WCMDPs.
Among them, \citep{XioWanLi_22} considers the most general setting, but is still limited to a single-budget constraint, state-independent costs, and typed heterogeneity.
To make a fair comparison, we evaluate both policies under this special case, while keeping all other settings the same as in the first set of experiments.
ID policy turns out to have a slight improvement over ERC policy in some instances, one of which is shown in Figure~\ref{fig:single-constr}, and has comparable performance in others. 
Importantly, unlike ERC, the ID policy applies to far more general classes of WCMDPs.

\bibliography{refs-weina}

@string{aap = "Adv. Appl. Probab."}

@string{automatica = "Automatica"}

@string{cdc = "Proc. IEEE Conf. Decision and Control (CDC)"}

@string{ieeetac = "IEEE Trans. Autom. Control"}

@string{japrob = "J. Appl. Probab."}

@string{kdd = "Proc. Ann. ACM SIGKDD Conf. Knowledge Discovery and Data Mining (KDD)"}

@string{ques = "Queueing Syst."}

@string{webconf = "Proc. ACM Web Conf."}

@string{performance = "Proc. Int. Symp. Computer Performance, Modeling, Measurements and Evaluation (IFIP Performance)"}

@string{sigmetricsper = "ACM SIGMETRICS Perform. Evaluation Rev."}

@string{neurips = "Conf. Neural Information Processing Systems (NeurIPS)"}

@string{peis = "Probab. Eng. Inf. Sci."}

@string{ijcai = "Proc. Int. Jt. Conf. Artificial Intelligence (IJCAI)"}

@string{anap = "Ann. Appl. Probab."}

@string{ms="Manage. Sci."}

@string{or="Oper. Res."}

@string{euor="Eur. J. Oper. Res."}

@string{moor="Math. Oper. Res."}

@string{ss="Stoch. Syst."}

@string{uai="Conf. Uncertainty in Artificial Intelligence (UAI)"}

@string{pomacs="Proc. ACM Meas. Anal. Comput. Syst."}

@string{sc="Statist. Sci."}

@string{mathematics="Mathematics"}

@string{aamas="Proc. Int. Conf. Autonomous Agents and Multiagent Systems (AAMAS)"}

@string{jrssB="J. Roy. Stat. Soc. B Met."}

@article{ZhaHonWan_25_het_wcmdp_nips_ver_on_arxiv,
  title={Projection-based {Lyapunov} method for fully heterogeneous weakly-coupled {MDPs}},
  author={Zhang, Xiangcheng and Hong, Yige and Wang, Weina},
  journal={arXiv:2502.06072v6 {[cs.LG]}},
  year={2025}
}

@misc{ZhaHonWan_25_github,
  author = {Zhang, Xiangcheng and Hong, Yige and Wang, Weina},
  title = {Projection-based {Lyapunov} method for fully heterogeneous weakly-coupled {MDPs}},
  year = {2025},
  publisher = {GitHub},
  journal = {GitHub repository},
  howpublished = {GitHub repository, \url{https://github.com/YigeHong/wcmdp-fully-hetero/}}
}

@article{ZhaMuk_24_rate_mat_prun,
author = {Zhao, Zhisheng and Mukherjee, Debankur},
title = {Optimal Rate-Matrix Pruning For Heterogeneous Systems},
year = {2024},
issue_date = {March 2024},
publisher = {Association for Computing Machinery},
address = {New York, NY, USA},
volume = {51},
number = {4},
optissn = {0163-5999},
opturl = {https://doi.org/10.1145/3649477.3649492},
optdoi = {10.1145/3649477.3649492},
journal = sigmetricsper,
optmonth = feb,
pages = {26–27},
numpages = {2}
}

@article{ZhaMukWu_24_data_loc_het,
author = {Zhao, Zhisheng and Mukherjee, Debankur and Wu, Ruoyu},
title = {Exploiting Data Locality to Improve Performance of Heterogeneous Server Clusters},
journal = ss,
volume = {14},
number = {3},
pages = {229-272},
year = {2024},
optdoi = {10.1287/stsy.2022.0040},
OPTURL = {https://doi.org/10.1287/stsy.2022.0040},
eprint = {https://doi.org/10.1287/stsy.2022.0040}
}

@article{GolAvr_24_wcmdp_multichain,
      title={Asymptotically Optimal Policies for Weakly Coupled {Markov} Decision Processes}, 
      author={Diego Goldsztajn and Konstantin Avrachenkov},
      year={2024},
      journal={arXiv:2406.04751 {[math.OC]}},
      opteprint={2406.04751},
      optarchivePrefix={arXiv},
      optprimaryClass={math.OC},
      opturl={https://arxiv.org/abs/2406.04751}
}

@book{Dur_19_prob_book, 
    place={Cambridge}, 
    edition={5}, 
    optseries={Cambridge Series in Statistical and Probabilistic Mathematics}, 
    title={Probability: Theory and Examples}, 
    OPTDOI={10.1017/9781108591034}, 
    publisher={Cambridge University Press}, 
    author={Durrett, Rick}, 
    year={2019}, 
    collection={Cambridge Series in Statistical and Probabilistic Mathematics}}

@article{Zhang24_het,
  title={Leveraging Nondegeneracy in Dynamic Resource Allocation},
  author={Zhang, Jingwei},
  journal={Available at SSRN},
  year={2024}
}

@article{AllGas22_het,
author = {Allmeier, Sebastian and Gast, Nicolas},
title = {Mean Field and Refined Mean Field Approximations for Heterogeneous Systems: It Works!},
year = {2022},
issue_date = {March 2022},
publisher = {Association for Computing Machinery},
address = {New York, NY},
volume = {6},
number = {1},
opturl = {https://doi.org/10.1145/3508033},
optdoi = {10.1145/3508033},
journal = pomacs,
month = feb,
articleno = {13},
numpages = {43}
}

@article{AllGas22_graphon,
author = {Allmeier, Sebastian and Gast, Nicolas},
title = {Accuracy of the Graphon Mean Field Approximation for Interacting Particle Systems},
journal = ss, 
volume = {15},
number = {4},
pages = {273-290},
year = {2025},
optdoi = {10.1287/stsy.2024.0070}
}

@article{HonXieChe_24,
  author  = {Hong, Yige and Xie, Qiaomin and Chen, Yudong and Wang, Weina},
  title   = {Unichain and Aperiodicity Are Sufficient for Asymptotic Optimality of Average-Reward Restless Bandits},
  journal = moor,
  year    = {2025},
  doi     = {10.1287/moor.2024.0678},
  note    = {Articles in Advance}
}

@inproceedings{HonXieChe_23,
author = {Yige Hong and Qiaomin Xie and Yudong Chen and Weina Wang},
booktitle = neurips,
title = {Restless Bandits with Average Reward: Breaking the Uniform Global Attractor Assumption},
year = 2023,
}

@article{hong2024exponential,
  title={Achieving Exponential Asymptotic Optimality in Average-Reward Restless Bandits without Global Attractor Assumption},
  author={Hong, Yige and Xie, Qiaomin and Chen, Yudong and Wang, Weina},
  journal={arXiv:2405.17882 {[cs.LG]}},
  year={2024}
}

@book{LatSze_20,
    place={Cambridge},
    title={Bandit Algorithms},
    publisher={Cambridge University Press},
    author={Lattimore, Tor and Szepesvári, Csaba},
    year={2020}
}

@article{ZhaFra_21,
author = "Zhang, Xiangyu and Frazier, Peter I.",
title = "Restless Bandits with Many Arms: Beating the Central Limit Theorem",
journal = "arXiv:2107.11911 [math.OC]",
year = 2021,
month = jul,
}

@book{Put_05,
  title={{Markov} decision processes: Discrete stochastic dynamic programming},
  author={Puterman, Martin L},
  year={2005},
  publisher={John Wiley \& Sons}
}

@book{GitGlaWeb_11,
author = {Gittins, John and Kevin Glazebrook and Richard Weber},
title = {Multi-armed bandit allocation indices},
publisher = {John Wiley \& Sons},
year = 2011,
}

@article{ZhaFra_22_discounted_rb,
      title={Near-optimality for infinite-horizon restless bandits with many arms}, 
      author={Xiangyu Zhang and Peter I. Frazier},
      year=2022,
      journal="arXiv:2203.15853 [cs.LG]"
}

@article{GlaMitAns_05_rb_repair,
	author = {K. D. Glazebrook and H. M. Mitchell and P. S. Ansell},
	optdoi = {https://doi.org/10.1016/j.ejor.2004.01.036},
	optissn = {0377-2217},
	journal = euor,
	number = {1},
	pages = {267-284},
	title = {Index policies for the maintenance of a collection of machines by a set of repairmen},
	opturl = {https://www.sciencedirect.com/science/article/pii/S0377221704000876},
	volume = {165},
	year = {2005}
}

@article{HuFra_17_rb_asymptotic,
      title={An Asymptotically Optimal Index Policy for Finite-Horizon Restless Bandits}, 
      author={Weici Hu and Peter Frazier},
      year={2017},
      journal="arXiv:1707.00205 [math.OC]"
}

@article{ZayJasWan_19_rb,
        author = {Zayas-Cabán, Gabriel and Jasin, Stefanus and Wang, Guihua},
        year = {2019},
        pages = {745-772},
        title = {An asymptotically optimal heuristic for general nonstationary finite-horizon restless multi-armed, multi-action bandits},
        number = {3},
        volume = {51},
        journal = aap,
        optdoi = {10.1017/apr.2019.29}
}

@article{BroSmi_19_rb,
        author = {Brown, David B. and Smith, James E.},
        title = {Index Policies and Performance Bounds for Dynamic Selection Problems},
        journal = ms,
        volume = {66},
        number = {7},
        pages = {3029-3050},
        year = 2020,
        OPTdoi = {10.1287/mnsc.2019.3342},
        OPTURL = {https://doi.org/10.1287/mnsc.2019.3342},
}

@article{PapTsi_99_pspace,
     OPTURL = {http://www.jstor.org/stable/3690486},
     author = {Christos H. Papadimitriou and John N. Tsitsiklis},
     journal = moor,
     number = {2},
     pages = {293--305},
     title = {The Complexity of Optimal Queuing Network Control},
     volume = {24},
     year = {1999}
}

@article{Whi_88_rb,
  title={Restless bandits: activity allocation in a changing world},
  author={Peter Whittle},
  journal=japrob,
  year={1988},
  volume={25},
  pages={287 - 298}
}

@article{WebWei_90,
     OPTURL = {http://www.jstor.org/stable/3214547},
     author = {Richard R. Weber and Gideon Weiss},
     journal = japrob,
     number = {3},
     pages = {637--648},
     title = {On an Index Policy for Restless Bandits},
     volume = {27},
     year = {1990}
}

@article{Ver_16_verloop,
     OPTURL = {http://www.jstor.org/stable/24810047},
     author = {I. M. Verloop},
     journal = anap,
     number = {4},
     pages = {1947--1995},
     title = {Asymptotically Optimal Priority Policies for Indexable and Nonindexable Restless Bandits},
     volume = {26},
     year = {2016}
}

@article{GasGauYan_23_exponential,
author = {Gast, Nicolas and Gaujal, Bruno and Yan, Chen},
title = {Linear Program-Based Policies for Restless Bandits: Necessary and Sufficient Conditions for (Exponentially Fast) Asymptotic Optimality},
journal = moor,
volume = {49},
number = {4},
pages = {2468-2491},
year = {2024},
OPTdoi = {10.1287/moor.2022.0101},
OPTURL = {https://doi.org/10.1287/moor.2022.0101},
}

@article{GasGauYan_23_whittles,
author={Gast, Nicolas and Gaujal, Bruno and Yan, Chen}, 
title="Exponential asymptotic optimality of {Whittle} index policy", 
journal = ques,
year = 2023,
pages = {107--150},
volume = {104},
OPTurl={https://doi.org/10.1007/s11134-023-09875-x}
}

@article{GasGauYan_24_reopt,
      title={Reoptimization Nearly Solves Weakly Coupled {Markov} Decision Processes}, 
      author={Nicolas Gast and Bruno Gaujal and Chen Yan},
      year={2024},
      journal={arXiv:2211.01961 {[math.OC]}},
      opteprint={2211.01961},
      optarchivePrefix={arXiv},
      optprimaryClass={math.OC}
}

@incollection{GitJon_74,
  address = {Amsterdam},
  author = {Gittins, J. C. and Jones, D. M.},
  booktitle = {Progress in Statistics},
  description = {bandit problems},
  editor = {Gani, J.},
  pages = {241-266},
  publisher = {North-Holland},
  title = {A Dynamic Allocation Index for the Sequential Design of Experiments},
  year = 1974
}

@article{Git_79,
  title={Bandit processes and dynamic allocation indices},
  author={Gittins, John C.},
  journal=jrssb,
  volume={41},
  number={2},
  pages={148--164},
  year={1979},
  publisher={Wiley Online Library}
}

@inproceedings{GhoNagJaiTam_23_finite_discount,
	author = {Ghosh, Abheek and Nagaraj, Dheeraj and Jain, Manish and Tambe, Milind},
	title = {Indexability is Not Enough for {Whittle}: Improved, Near-Optimal Algorithms for Restless Bandits},
	year = {2023},
	booktitle = aamas, 
	pages = {1294-1302},
	numpages = {9},
}

@phdthesis{Haw_03,
author = "Hawkins, Jeffrey Thomas",
title = "A {Langrangian} decomposition approach to weakly coupled dynamic optimization problems and its applications",
school = "Operations Research Center, Massachusetts Institute of Technology",
year = 2003,
}

@article{BroZha_23,
author = {Brown, David B. and Zhang, Jingwei},
title = {Fluid Policies, Reoptimization, and Performance Guarantees in Dynamic Resource Allocation},
journal = or,
volume = {73},
number = {2},
pages = {1029-1045},
year = {2025},
OPTdoi = {10.1287/opre.2022.0601},
OPTURL = {https://doi.org/10.1287/opre.2022.0601},
}

@article{AvrBor_22,
	title = {{Whittle} index based {Q}-learning for restless bandits with average reward},
	journal = automatica,
	volume = {139},
	pages = {110186},
	year = 2022,
	OPTdoi = {https://doi.org/10.1016/j.automatica.2022.110186},
	OPTurl = {https://www.sciencedirect.com/science/article/pii/S0005109822000310},
	author = {Konstantin E. Avrachenkov and Vivek S. Borkar},
}

@inproceedings{XioWanLi_22,
	author = {Xiong, Guojun and Wang, Shufan and Li, Jian},
	booktitle = neurips,
	pages = {17911--17925},
	title = {Learning Infinite-Horizon Average-Reward Restless Multi-Action Bandits via Index Awareness},
	OPTurl = {https://proceedings.neurips.cc/paper_files/paper/2022/file/71f003060ce1e8b6b4856023b67cda5d-Paper-Conference.pdf},
	optvolume = {35},
	year = 2022,
}

@INPROCEEDINGS{Yan_24,
  author={Yan, Chen},
  booktitle=cdc, 
  title={An Optimal-Control Approach to Infinite-Horizon Restless Bandits: Achieving Asymptotic Optimality with Minimal Assumptions}, 
  year={2024},
  volume={},
  number={},
  pages={6665-6672},
  optdoi={10.1109/CDC56724.2024.10886365}
}

@article{HodGla_15, 
title={On the asymptotic optimality of greedy index heuristics for multi-action restless bandits}, 
volume={47}, 
OPTDOI={10.1239/aap/1444308876}, 
number={3}, 
journal=aap, 
author={Hodge, D. J. and Glazebrook, K. D.}, 
year={2015}, 
pages={652–667}
}

@inproceedings{NakHou_22,
 author = {Nakhleh, Khaled and Hou, I-Hong},
 booktitle = neurips,
 pages = {28734--28746},
 title = {{DeepTOP}: Deep Threshold-Optimal Policy for {MDPs} and {RMABs}},
 OPTurl = {https://proceedings.neurips.cc/paper_files/paper/2022/file/b8bf2c0dd0b48511889b7d3b2c5fc8f5-Paper-Conference.pdf},
 optvolume = {35},
 year = {2022}
}

@inproceedings{NakGanHsi_21,
 author = {Nakhleh, Khaled and Ganji, Santosh and Hsieh, Ping-Chun and Hou, I-Hong and Shakkottai, Srinivas},
 booktitle = neurips,
 pages = {828--839},
 title = {{NeurWIN}: Neural {Whittle} Index Network For Restless Bandits Via Deep {RL}},
 OPTurl = {https://proceedings.neurips.cc/paper_files/paper/2021/file/0768281a05da9f27df178b5c39a51263-Paper.pdf},
 optvolume = {34},
 year = {2021}
}

@article{BroZha_22,
	author = {Brown, David B. and Zhang, Jingwei},
	title = {Dynamic Programs with Shared Resources and Signals: Dynamic Fluid Policies and Asymptotic Optimality},
	journal = or,
	volume = {70},
	number = {5},
	pages = {3015-3033},
	year = {2022},
	OPTdoi = {10.1287/opre.2021.2181},
	OPTURL = {https://doi.org/10.1287/opre.2021.2181},
}

@article{DaeChoGri_23,
	author = {D’Aeth, Josh C. and Ghosal, Shubhechyya and Grimm, Fiona and Haw, David and Koca, Esma and Lau, Krystal and Liu, Huikang and Moret, Stefano and Rizmie, Dheeya and Smith, Peter C. and Forchini, Giovanni and Miraldo, Marisa and Wiesemann, Wolfram},
	title = {Optimal Hospital Care Scheduling During the {SARS-CoV-2} Pandemic},
	journal = ms,
	volume = {69},
	number = {10},
	pages = {5923-5947},
	year = {2023},
	OPTdoi = {10.1287/mnsc.2023.4679},
	OPTURL = {https://doi.org/10.1287/mnsc.2023.4679},
}

@article{Whi_80,
	author = {Whittle, Peter},
	title = {Multi-Armed Bandits and the {Gittins} Index},
	journal = jrssB,
	volume = {42},
	number = {2},
	pages = {143-149},
	year = {1980},
	month = {12},
	OPTdoi = {10.1111/j.2517-6161.1980.tb01111.x},
	OPTurl = {https://doi.org/10.1111/j.2517-6161.1980.tb01111.x},
}

@ARTICLE{VarWalBuy_85,
	author={Varaiya, P. and Walrand, J. and Buyukkoc, C.},
	journal=ieeetac, 
	title={Extensions of the multiarmed bandit problem: The discounted case}, 
	year={1985},
	volume={30},
	number={5},
	pages={426-439},
	OPTdoi={10.1109/TAC.1985.1103989}
}

@article{Web_92,
	author = {Richard Weber},
	title = {On the {Gittins} Index for Multiarmed Bandits},
	volume = {2},
	journal = anap,
	number = {4},
	pages = {1024 -- 1033},
	year = {1992},
	OPTdoi = {10.1214/aoap/1177005588},
	OPTURL = {https://doi.org/10.1214/aoap/1177005588}
}

@article{Tsi_94,
	author = {John N. Tsitsiklis},
	title = {A Short Proof of the {Gittins} Index Theorem},
	volume = {4},
	journal = anap,
	number = {1},
	pages = {194 -- 199},
	year = {1994},
	OPTdoi = {10.1214/aoap/1177005207},
	OPTURL = {https://doi.org/10.1214/aoap/1177005207}
}

@article{BerNin_96,
	author = {Bertsimas, Dimitris and Ni\~{n}o-Mora, Jos\'{e}},
	title = {Conservation Laws, Extended Polymatroids and Multiarmed Bandit Problems; A Polyhedral Approach to Indexable Systems},
	journal = moor,
	volume = {21},
	number = {2},
	pages = {257-306},
	year = {1996},
	OPTdoi = {10.1287/moor.21.2.257},
	OPTURL = {https://doi.org/10.1287/moor.21.2.257},
}

@inproceedings{zhou2023rl,
  title={{RL-MPCA}: A reinforcement learning based multi-phase computation allocation approach for recommender systems},
  author={Zhou, Jiahong and Mao, Shunhui and Yang, Guoliang and Tang, Bo and Xie, Qianlong and Lin, Lebin and Wang, Xingxing and Wang, Dong},
  booktitle=webconf,
  pages={3214--3224},
  year={2023}
}

@inproceedings{boutilier2016budget,
  title={Budget Allocation using Weakly Coupled, Constrained {Markov} Decision Processes},
  author={Boutilier, Craig and Lu, Tyler},
  booktitle=uai,
  year={2016},
  pages = {52-61},
  optaddress	= {New York, NY}
}

@article{yu2018deadline,
  title={Deadline scheduling as restless bandits},
  author={Yu, Zhe and Xu, Yunjian and Tong, Lang},
  journal=ieeetac,
  volume={63},
  number={8},
  pages={2343--2358},
  year={2018},
  publisher={IEEE}
}

@inproceedings{biswas2021learning,
  title={Learning Index Policies for Restless Bandits with Application to Maternal Healthcare},
  author={Biswas, Arpita and Aggarwal, Gaurav and Varakantham, Pradeep and Tambe, Milind},
  booktitle=ijcai,
  pages={1467--1468},
  year={2021}
}

@article{villar2016indexability,
  title={Indexability and optimal index policies for a class of reinitialising restless bandits},
  author={Villar, Sof{\'\i}a S},
  journal=peis,
  volume={30},
  number={1},
  pages={1--23},
  year={2016},
  publisher={Cambridge University Press}
}

@inproceedings{killian2022restless,
  title={Restless and uncertain: Robust policies for restless bandits via deep multi-agent reinforcement learning},
  author={Killian, Jackson A and Xu, Lily and Biswas, Arpita and Tambe, Milind},
  booktitle=uai,
  pages={990--1000},
  year={2022},
}

@article{el2023weakly,
  title={Weakly Coupled Deep {Q}-Networks},
  author={El Shar, Ibrahim and Jiang, Daniel},
  journal=neurips,
  volume={36},
  year=2023,
}

@inproceedings{killian2021q_multi,
  title={Q-learning {Lagrange} policies for multi-action restless bandits},
  author={Killian, Jackson A and Biswas, Arpita and Shah, Sanket and Tambe, Milind},
  booktitle=kdd,
  pages={871--881},
  year={2021}
}

@article{robledo2022qwi,
  title={{QWI}: Q-learning with {Whittle} index},
  author={Robledo, Francisco and Borkar, Vivek and Ayesta, Urtzi and Avrachenkov, Konstantin},
  journal=sigmetricsper,
  volume={49},
  number={2},
  pages={47--50},
  year={2022},
  publisher={ACM New York, NY, USA}
}

@inproceedings{biswas21index_healthcare,
author = {Biswas, Arpita and Aggarwal, Gaurav and Varakantham, Pradeep and Tambe, Milind},
title = {Learning Index Policies for Restless Bandits with Application to Maternal Healthcare},
year = {2021},
optisbn = {9781450383073},
optpublisher = {International Foundation for Autonomous Agents and Multiagent Systems},
optaddress = {Richland, SC},
booktitle = aamas,
pages = {1467–1468},
numpages = {2},
optlocation = {Virtual Event, United Kingdom},
optseries = {AAMAS '21}
}

@inproceedings{XiongLi23,
 author = {Xiong, Guojun and Li, Jian},
 booktitle = neurips,
 pages = {29048--29073},
 title = {Finite-Time Analysis of {Whittle} Index based {Q}-Learning for Restless Multi-Armed Bandits with Neural Network Function Approximation},
 OPTurl = {https://proceedings.neurips.cc/paper_files/paper/2023/file/5c7c66dfc9f93f0c738947f3b1c13832-Paper-Conference.pdf},
 optvolume = {36},
 year = 2023,
}

\appendix

\newpage
\section{Detailed review on related work}
\label{app:additional-related-work}
In this section, we provide a more detailed, though still non-exhaustive, review of the literature. 
We mainly focus on theoretical work with formal performance guarantees, leaving out the extensive body of work with empirical results.
We begin by surveying papers with the same reward criterion as ours, i.e., infinite-horizon average-reward criterion. 
In this setting, we first review existing work on restless bandits (RBs), which is an extensively studied special case of WCMDPs. 
We then give a more detailed review of existing results on WCMDPs.
Next, we turn to work that considers other reward criteria: the finite-horizon total-reward criterion and the infinite-horizon discounted-reward criterion. 
Finally, we briefly mention other problems that are related to WCMDPs.

\paragraph{Infinite-horizon average-reward RBs.}
For the homogeneous arm setting, the first asymptotic optimality result for average-reward homogeneous RBs is established by \citet{WebWei_90}: they show that the \emph{Whittle index policy} \citep{Whi_88_rb} achieves an $o(1)$ optimality gap as the number of arms $N$ goes to infinity. There are three key assumptions in \citep{WebWei_90}: indexability, the global attractor property, and the aperiodic-unichain condition.
These assumptions have been gradually relaxed in the subsequent papers. In particular, \cite{Ver_16_verloop} proposes a class of priority policies based on an LP relaxation. This class of policies, later referred to as the \emph{LP-Priority policies}, generalizes the Whittle index policy. Each LP-Priority policy achieves an $o(1)$ optimality gap without requiring indexability.
The work \citep{HonXieChe_23} is the first one that breaks the global attractor property assumption.
The authors propose a policy named Follow-the-Virtual-Advice (FTVA), which achieves an $O(1/\sqrt{N})$ optimality gap under an assumption named the Synchronization Assumption;
there exist problem instances that satisfy the Synchronization Assumption but do not satisfy the global attractor property.
Later work \citep{HonXieChe_24} further relaxes the conditions and only requires the aperiodic-unichain condition to achieve an $O(1/\sqrt{N})$ optimality gap.
More recently, \citet{Yan_24} proposes the \emph{align-and-steer policy}, which further weakens the aperiodic-unichain condition and achieves an $o(1)$ optimality gap.

Parallel to relaxing the assumptions for asymptotic optimality, another line of work has focused on improving the optimality gap beyond $O(1/\sqrt{N})$ under slightly stronger assumptions \citep{GasGauYan_23_exponential,GasGauYan_23_whittles,hong2024exponential}.
Specifically, \citet{GasGauYan_23_whittles} show that the Whittle index policy has an $O(\exp(-cN))$ optimality gap for some constant $c>0$. In addition to indexability and the aperiodic-unichain condition, \citep{GasGauYan_23_whittles} also requires a stronger version of the global attractor property named Uniform Global Attractor Property (UGAP), and a condition called non-singularity.
Subsequently, \citet{GasGauYan_23_exponential} show that LP-Priority policies achieve $O(\exp(-cN))$ optimality gaps assuming the aperiodic-unichain condition, UGAP, and a non-degenerate condition that is equivalent to non-singularity.
More recently, \citet{hong2024exponential} propose a \emph{two-set policy} that also achieves an $O(\exp(-cN))$ optimality gap while replacing UGAP of \citep{GasGauYan_23_exponential} with a much weaker condition named local stability.

Among the aforementioned work, \cite{Ver_16_verloop} and \cite{HonXieChe_23} have addressed the heterogeneous arm setting.
The setting studied in \citep{Ver_16_verloop} is the typed heterogeneous setting, where the $N$ arms are divided into a constant number of types as $N\to\infty$.  
The paper \citep{HonXieChe_23} includes an extension to the fully heterogeneous setting.
In particular, the proposed FTVA policy generalizes to the fully heterogeneous setting and leads to an optimality gap of $O(\overline{\tau}^{\text{sync}}_{\max}/\sqrt{N})$, where $\overline{\tau}^{\text{sync}}_{\max}$ is the maximum of a quantity called the synchronization time across all arms.
Therefore, for this result to yield asymptotic optimality, there needs to be a further assumption that $\overline{\tau}^{\text{sync}}_{\max}=o(\sqrt{N})$.
We re-emphasize that the FTVA policy does not generalize to WCMDPs.
The main reason is that FTVA heavily relies on the fact that an RB only constrains the number of pulls, while a WCMDP has budget constraints on cost functions each depending on both the state and action of an arm.

\paragraph{Infinite-horizon average-reward WCMDPs.}
Work on average-reward WCMDPs remains relatively scarce, and to our knowledge, fully heterogeneous WCMDPs have yet to be addressed. 
Compared to RBs, a WCMDP allows multiple actions for each arm and multiple cost constraints, where each cost function is a function of both the state and the action of an arm.
The line of research \cite{HodGla_15,Ver_16_verloop,XioWanLi_22,GolAvr_24_wcmdp_multichain} has generalized the action space and cost model of RBs to WCMDPs, and some of them allow for typed heterogeneity.
In particular, \citet{HodGla_15} generalize the Whittle index policy to homogeneous WCMDPs with a single constraint and multiple actions, where each action represents a different activation level and has a different cost.
\citet{Ver_16_verloop} extends the LP-Priority policies to typed heterogeneous WCMDPs with a single constraint and multiple actions, but requires each action to have the same cost.
Then \citet{XioWanLi_22} propose another index policy for typed heterogeneous WCMDPs with a single constraint, and allow each action to have a different cost.
In the three papers above \citep{HodGla_15,Ver_16_verloop,XioWanLi_22}, $o(1)$ optimality gaps have been proved under a similar set of assumptions as in most restless bandit papers, i.e., aperiodic-unichain (or irreducibility) condition, global attractor property, and a generalized indexability condition if the policy is Whittle index.
Finally, \citet{GolAvr_24_wcmdp_multichain} consider homogeneous WCMDPs with multiple actions and multiple cost constraints with general cost functions,
and propose a class of policies with $o(1)$ optimality gaps under a weaker-than-standard aperiodic-unichain condition.

\paragraph{Finite-horizon total-reward RBs and WCMDPs.}
Next, we review the asymptotic optimality results for finite-horizon total-reward RBs and WCMDPs. The finite-horizon setting is better understood than the average-reward setting, partly because the analysis in the finite horizon is not hindered by the technical conditions arising in average-reward MDPs, such as the unichain condition and the global attractor property. On the other hand, computing asymptotically optimal policies for the finite-horizon setting is more complicated, requiring a careful optimization of the transient sample paths.

\citet{HuFra_17_rb_asymptotic} propose the first asymptotically optimal policy for finite-horizon homogeneous RBs, which achieves an $o(1)$ optimality gap without any assumptions.\footnote{Here, we measure the optimality gap in terms of the reward per arm, to be consistent with our convention. However, in the papers on the finite-horizon total-reward setting, it is also common to measure the optimality gap in terms of the total reward of all arms, which differs from ours by a factor of $N$. We also adopt the same convention when reviewing the papers on the infinite-horizon discounted-reward setting.}
Since then, researchers have established asymptotic optimality in more general settings \citep{ZayJasWan_19_rb,DaeChoGri_23,GhoNagJaiTam_23_finite_discount,BroSmi_19_rb,BroZha_23}.
Among these papers, the most general setting is addressed by \citet{BroZha_23}, where the authors consider fully heterogeneous WCMDPs; they obtain $O(1/\sqrt{N})$ optimality gaps for a naive fluid policy and a reoptimization-based fluid policy, among a few other results to be reviewed in the next paragraphs.
Notably, there is also a further generalization of fully heterogeneous WCMDPs, which involves an exogenous state that affects all arms' transitions, rewards, and constraints; \citet{BroZha_22} propose this setting, where they achieve an $O(1/\sqrt{N})$ optimality gap using a dynamic fluid policy.

Another line of work has improved the optimality gap beyond the order $O(1/\sqrt{N})$ by making an additional assumption called non-degeneracy. Specifically, \citet{ZhaFra_21} establish an $O(1/N)$ optimality gap in non-degenerate homogeneous RBs. \citet{GasGauYan_23_exponential} then propose a different policy for the same setting that improves the optimality gap to $O(\exp(-cN))$.
Later, \citet{GasGauYan_24_reopt} and \citet{BroZha_23} establish $O(1/N)$ optimality gaps for homogeneous and typed heterogeneous WCMDPs, respectively, assuming non-degeneracy.
More recently, \citet{Zhang24_het} proposes a policy for fully heterogeneous WCMDPs; the optimality gap bound of the policy interpolates between $O(1/\sqrt{N})$ and $O(1/N)$ as the degree of non-degeneracy varies, unifying the performance bounds in the degenerate and non-degenerate cases.

Despite the generality of the settings and the fast diminishing rate of the optimality gaps as $N \to \infty$, most of the optimality gaps in the finite-horizon setting depend \emph{super-linearly} on the time horizon, so they do not carry over to the infinite-horizon average-reward setting.
There are two exceptions, \citep{BroZha_23,GasGauYan_24_reopt}, which achieve optimality gaps that depend linearly on the time horizon under some special conditions: \citep{BroZha_23} requires all entries of the transition kernels to be bounded away from zero; \citep{GasGauYan_24_reopt} assumes an ergodicity property, which requires two arms in any different states to synchronize in a fixed number of steps under \emph{any sequence of actions} with a positive probability.
However, without these conditions, the optimality gaps in \citep{BroZha_23,GasGauYan_24_reopt} depend quadratically on the time horizon.
Apart from having distinct optimality gap bounds, all existing algorithms in the finite-horizon setting need to (sometimes repeatedly) solve LPs whose number of variables scales with the time horizon, so they cannot be directly adapted to the infinite-horizon average-reward setting.

\paragraph{Infinite-horizon discounted-reward RBs and WCMDPs.}
Asymptotic optimality has also been established for RBs and WCMDPs under the infinite-horizon discounted-reward criterion.
In particular, \citet{BroSmi_19_rb} establish an $O(N^{\log_2 (\sqrt{\gamma})})$ optimality gap for fully heterogeneous WCMDPs when $\gamma \in (1/2, 1)$.
Subsequently, \citet{ZhaFra_22_discounted_rb,GhoNagJaiTam_23_finite_discount} obtain $O(1/\sqrt{N})$ optimality gaps for homogeneous and typed heterogeneous RBs, and \citet{BroZha_23} establish the same order of optimality gap for fully heterogeneous WCMDPs.
Similar to the finite-horizon setting, most of these optimality gaps depend super-linearly on the effective time horizon $1/(1-\gamma)$ unless special conditions hold \citep{BroZha_23}, so they do not carry over to the infinite-horizon average-reward setting.
The policies here also require solving LPs whose complexities scale with the effective time horizon.

\paragraph{Restful bandits, stochastic multi-armed bandits.}
A special case of RB is the restful bandit (also referred to as nonrestless bandits, rested bandits, or Markovian bandits), where an arm's state does not change if it is not pulled.
The restful bandit problem has been widely studied, where the celebrated Gittins index policy is proven to be optimal \citep{GitJon_74,Git_79,BerNin_96,Tsi_94,Web_92,VarWalBuy_85,Whi_80}. We refer the readers to \citep{GitGlaWeb_11} for a comprehensive review of Gittins index and restful bandits. 
Another related topic is the stochastic multi-armed bandit (MAB) problem, which has been extensively studied; see the book \citep{LatSze_20} for a comprehensive overview. 
The key distinction between MABs and RBs is that arms are stateless in MABs, but stateful in RBs. Consequently, MAB becomes trivial with known model parameters, whereas RB is still non-trivial.

\section{Experimental details}
\label{app:details}
In this appendix, we provide details of the two WCMDP instances considered in \Cref{sec:experiments}, and the definition of the baseline policy, the ERC policy from \citep{XioWanLi_22}. 
The complete code for these experiments is available on GitHub \citep{ZhaHonWan_25_github}, and all results can be reproduced within 24 hours on a standard PC (e.g., 6-Core Intel Core i7).

\subsection{WCMDP instance generation}
\label{app:exp_details}

\paragraph{Details of the WCMDP instance 1 (In Figure~\ref{fig:multi-constr}).}
$|\sspa| = 10$, $|\aspa| = 4$, $K = 4$.
For each $i\in[N]$, $s\in\sspa$, and $a\in\aspa$, $r_i(s,0) = 0$, and $r_i(s,a)$ is independently sampled from the uniform distribution over $[0,1]$ for each $a\neq 0$; $P_{i}(s,a,\cdot)$ is independently sampled from the uniform distribution over the probability simplex. 
As for the cost function, for each $i\in[N]$, $s\in\sspa$, $a\in\aspa$, and $k\in[K]$, we have $c_{k,i}(s,0)=0$, and $c_{k,i}(s,a)$ is independently sampled from the uniform distribution over $[0,1]$ for each $a\neq 0$. 
For each $k\in[K]$, $\alpha_k$ is uniformly sampled from $\{0.05, 0.1, 0.15, \dots, 0.45\}$, i.e., the uniform distribution over integer multiples of $0.05$ in the interval $(0,0.5)$.

\paragraph{Details of the WCMDP instance 2 (In Figure~\ref{fig:single-constr}).}
This instance is typed heterogeneous with $10$ types, with equal fraction of arms in each type. Each arm has $|\sspa| = 10$ states, $|\aspa| = 4$ actions. 
For each type, the reward function and transition kernel are generated in the same way as each arm in ithe first instance. The budget is also sampled from the same distribution as the first instance. 
There is a single budget constraint. The cost function depends only on the action, consistent with the setting in \citep{XioWanLi_22}; the cost of each of action $a\in\aspa$ is sampled from the uniform distribution over $[0,1]$.

\subsection{Definition of ERC policy}
\label{app:erc_details}

The ERC policy \citep{XioWanLi_22} solves the same LP in \eqref{eq:lp} to define the single-armed policies $\pibs_i(a|s)$ as in \eqref{eq:optimal-single-armed-policy}. At each time step, the policy computes an index for each arm-state pair $(i,s) \in [N] \times \sspa$:
\[
    \mathcal{I}(i,s) = \sum_{a\in\aspa} \pibs_i(a|s) r_i(s,a).
\]
It then iterates through the arms in descending order of these indices. For each arm, it samples an action $a \sim \pibs_i(\cdot \mid s)$ and applies it only if the budget is sufficient. If not, it defaults to action 0. %

\section{Proving the LP relaxation}\label{app:proof-lem-lp-upper-bound}
In this section, we prove Lemma~\ref{lem:lp upper bound}, which shows that the linear program in \eqref{eq:lp} is a relaxation of the WCMDP problem. \Cref{lem:lp upper bound} is restated as follows.  

\lprelaxation*

\begin{proof}
To upper bound the optimal reward of the WCMDP, $R^*(N,\bm{S}_0)$, we observe that standard MDP theory ensures that a stationary Markovian policy achieves the optimal reward, as the WCMDP has finitely many system states and system actions \citep[Theorem 9.18]{Put_05}. 
Therefore, it suffices to show that $R(\pi,\bm{S}_0) \leq \rrel$ for any stationary policy $\pi$ and initial system state $\bm{S}_0$. 

For any stationary policy $\pi$, consider the \emph{state-action frequency under $\pi$}, given by  
\begin{align*}
    y^\pi_i(s,a) = \lim_{T\to\infty} \frac{1}{T} \sum_{t=0}^{T-1} \EE\mbracket{\mathbf{1}\sets{S_{i,t}^\pi=s,A_{i,t}^\pi=a}}\,,~~\forall s\in\sspa, a\in\aspa, i\in[N]\,.
\end{align*}
where the limit is well-defined due to the stationarity of $\pi$. 
We argue that $y^\pi\triangleq (y_i^\pi(s,a))_{i\in[N] s\in\sspa, a\in\aspa}$ is a feasible solution to the LP relaxation in \eqref{eq:lp}, with objective value being $R(\pi, \bm{S}_0)$. Then $R(\pi,\bm{S}_0) \leq \rrel$ follows from the optimality of $\rrel$. 

To show that $y^\pi$ satisfies the budget constraints of the LP relaxation \eqref{eq:lp:budget-constraint}, we compute as follows: for any $s\in\sspa$, $a\in\aspa$ and constraint $k\in [K]$, we have
\begin{align*}
    \sumN c_{k,i}(s,a) y^\pi_i(s,a) 
    &= \sumN \sum_{s\in\sspa,a\in\aspa} c_{k,i}(s,a) \lim_{T\to\infty} \frac{1}{T} \sum_{t=0}^{T-1}\EE\mbracket{\mathbf{1}\sets{S_{i,t}^\pi=s,A_{i,t}^\pi=a}}\\
    &= \lim_{T\to\infty} \frac{1}{T}\EE\mbracket{ \sumN\sum_{s\in\sspa,a\in\aspa}c_{k,i}(S_{i,t}^\pi,A_{i,t}^\pi)}\\
    &\leq \alpha_k N\,,
\end{align*}
where the inequality follows from the fact that  under a feasible $N$-armed policy $\pi$, $\sum_{i\in[N]}\sum_{s\in\sspa,a\in\aspa} c_{k,i}(S_{i,t}^\pi,A_{i,t}^\pi) \leq \alpha_k N$ for each budget constraint $k\in[K]$.

Then we verify that $y^\pi$ satisfies the stationarity constraint of the LP relaxation \eqref{eq:lp:stationarity-constraint}: for any state $s\in\sspa$ and arm $i\in[N]$, we have
\begin{align*}
    \sum_{s'\in\sspa,a'\in\aspa}y_i^\pi(s',a')\PP_i(s \mid s',a')&=\lim_{T\to\infty}\frac{1}{T} \sum_{t=0}^{T-1} \sum_{s'\in\sspa,a'\in\aspa}P(S_{i,t}^\pi=s',A_{i,t}^\pi=a')\PP_i(s\mid s',a')\\
    &=\lim_{T\to\infty} \frac{1}{T} \sum_{t=0}^{T-1}P(S_{i,t+1}^\pi=s) \\
    &= \lim_{T\to\infty} \frac{1}{T} \sum_{t=1}^T P(S_{i,t}^\pi=s)\\
    &= \sum_{a\in\aspa} y^\pi_i(s,a)\,.
\end{align*}

We then argue that for each $i\in [N]$, $(y_i^\pi(s,a))_{s\in\sspa, a\in\aspa}$ is in the probability simplex of $\sspa\times \aspa$, as required by the last constraint in \eqref{eq:lp:probability-constraint}, which is obvious: for any $i\in [N]$ and $s\in\sspa, a\in\aspa$, we have $y^\pi_i(s,a)\geq 0$; for any $i\in[N]$, we have 
\begin{align*}
    \sum_{s\in\sspa,a\in\aspa}y_i^\pi(s,a)= \lim_{T\to\infty}\sum_{t=0}^{T-1} \EE\mbracket{\sum_{s,a}\mathbf{1}\sets{S_{i,t}^\pi=s,A_{i,t}^\pi=a}}=1\,.
\end{align*} Therefore, $y^\pi$ satisfies the constraints of the LP relaxation. 

Finally, we show that the objective value of $y^\pi$ equals $R(\pi,\bm{S}_0)$:
\begin{align*}
    \frac{1}{N}\sumN \sum_{s\in\sspa,a\in\aspa} y_i^\pi(s,a)r_i(s,a)
    &=\frac{1}{N}\sumN \sum_{s\in\sspa,a\in\aspa} r_i(s,a) \lim_{T\to\infty} \frac{1}{T}\sum_{t=0}^{T-1}\EE\mbracket{\mathbf{1}\sets{S_{i,t}^\pi=s,A_{i,t}^\pi=a}}\\
    &= \lim_{T\to\infty} \frac{1}{T} \sum_{t=0}^{T-1}\frac{1}{N}\sumN \EE[r_i(S_{i,t}^\pi,A_{i,t}^\pi)]\\
    &= R(\pi,\bm{S}_0)\,.
\end{align*}
Because $\rrel$ is the optimal value of the LP relaxation, we have $\rrel \geq R(\pi,\bm{S}_0)$ for any stationary policy $\pi$. Taking $\pi$ to be the optimal policy finishes the proof. 
\end{proof}

\section{More details and properties of the ID reassignment algorithm}\label{app:more-ID-reassignment}

In this section, we provide more details and properties of the ID reassignment algorithm in Algorithm~\ref{alg:id-assign}.
We first explain the rationale of the algorithm at a high level, and then give a key property of it in Lemma~\ref{lem:positiveC}.

Recall that for each arm $i\in[N]$ and each cost type $k\in[K]$, the expected cost under the optimal single-armed policy is defined as $C_{k,i}^* =\sum_{s\in\sspa,a\in\aspa}y_i^*(s,a)c_{k,i}(s,a)$.
Based on $C_{k,i}^*$'s, we divide the cost constraints into \emph{active} constraints and \emph{inactive} constraints as follows.
For each cost type $k\in[K]$, we say the type-$k$ cost constraint is \emph{active} if
$\sum_{i\in[N]}C_{k,i}^*\ge \frac{\alpha_k}{2}N,$
and \emph{inactive} otherwise.
Let $\cA\subseteq [K]$ denote the set of cost types corresponding to active constraints.

Now consider a subset $D\subseteq [N]$ of arms.
For each $k\in[K]$, let $C_k^*(D)=\sum_{i\in D}C_{k,i}^*$,
i.e., $C_k^*(D)$ is the total expected type-$k$ cost for arms in $D$ under the optimal single-armed policies.
In our analysis, we often need to consider a notion of \emph{remaining} budget, defined as
\begin{equation}\label{eq:remain-bud}
    \rembud_k(D)=
    \begin{cases}
        \alpha_k N-C_k^*(D),\quad&\text{if }k\in\cA,\\
        \alpha_k N-C_k^*(D)-\frac{\alpha_k}{3}|D|,\quad&\text{otherwise},
    \end{cases}
\end{equation}
where the $\frac{\alpha_k}{3}|D|$ is a correction term when type-$k$ constraint is inactive. 
Note that the budget constraint \eqref{eq:lp:budget-constraint} of the LP relaxation implies that $C_k^*(D) \leq \alpha_k N$. Consequently, $\rembud_k(D)\ge 0$ and $C_k^*(D)+\rembud_k(D)\le \alpha_k N$ for all $k\in[K]$ and all $D\subseteq [N]$.

The remaining budgets play a key role in the design of the ID reassignment algorithm.
Roughly speaking, the goal of the ID reassignment is to ensure that when we expand a set of arms from $[n_1]$ to $[n_2]$ for some $n_1\le n_2$, the drop in the remaining budget of any type $k$, i.e., $\rembud_k([n_1])-\rembud_k([n_2])$, is (almost) at least linear in $n_2-n_1$.
Note that this property is automatically satisfied for $k$ if the type-$k$ constraint is inactive.
This property is formalized in Lemma~\ref{lem:positiveC}, and the need for it will become clearer in \Cref{sec:proof-main} when we introduce the so-called focus set in our analysis.

To achieve this desired property, we design our ID reassignment algorithm in the following way.
If the set of active constraints is empty, i.e., $\cA=\emptyset$, then there is no need to reassign the IDs. 
Otherwise, i.e., when $\cA\neq \emptyset$, we first carefully choose two parameters, a positive real number $\costIDreassign$ and a positive integer $\gsizeIDreassign$.
We then divide the full ID set $[N]$ into groups of size $\gsizeIDreassign$, i.e., $[\gsizeIDreassign],[\gsizeIDreassign+1:2\gsizeIDreassign],[2\gsizeIDreassign+1:3\gsizeIDreassign],\dots,[(\lfloor N/\gsizeIDreassign\rfloor-1)\gsizeIDreassign+1: \lfloor N/\gsizeIDreassign\rfloor \gsizeIDreassign]$, and the remainder.
We ensure that after the reassignment, each group contains at least one arm $i$ with $C_{k,i}^*\ge \costIDreassign$ for each active constraint type $k\in\cA$.

The key here is to choose $\costIDreassign$ and $\gsizeIDreassign$ properly so such a reassignment is feasible.
In particular, we choose $\costIDreassign=\alpha_{\min}/4$, where $\alpha_{\min}=\min_{k\in[K]}\alpha_k$, and let
$\gsizeIDreassign = \left\lceil \frac{(c_{\max}-\costIDreassign)K}{\alpha_{\min}/2-\costIDreassign}\right\rceil.$
Note that $\gsizeIDreassign\ge K$ since one can verify that $c_{\max}\ge \alpha_{\min}/2$ when $\cA\neq \emptyset$.

\begin{restatable}{lemma}{positiveC}\label{lem:positiveC}
After applying the ID reassignment algorithm (Algorithm~\ref{alg:id-assign}), for any $n_1,n_2$ with $1\le n_1\le n_2\le N$, we have
\begin{equation}\label{eq:remaining-budget-decrease}
    \rembud_k([n_1])-\rembud_k([n_2])\ge \eta_c(n_2-n_1)-M_c,
\end{equation}
for all $k\in[K]$,
where $\eta_c>0$ and $M_c>0$ are constants determined by $\alpha_{\min}$, $c_{\max}$, and $K$.

Further, let $\rembud(D)=\min_{k\in[K]}\rembud_k(D)$ for all $D\subseteq [N]$.  Then the bound \eqref{eq:remaining-budget-decrease} implies that for any $n_1,n_2$ with $1\le n_1\le n_2\le N$,
\begin{equation}
    \rembud([n_1])-\rembud([n_2])\ge \eta_c(n_2-n_1)-M_c.
\end{equation}
\end{restatable}

\begin{proof}
Our goal is to prove that for any $n_1,n_2$ with $1\le n_1\le n_2\le N$, we have
\begin{equation}
    \rembud_k([n_1])-\rembud_k([n_2])\ge \eta_c(n_2-n_1)-M_c,
\end{equation}
for all $k\in[K]$, where
\begin{align*}
    \eta_c = \min\sets{\frac{\alpha_{\min}}{3},\costIDreassign\cdot\left(\left\lceil \frac{(c_{\max}-\costIDreassign)K}{\alpha_{\min}/2-\costIDreassign}\right\rceil\right)^{-1}}, \quad M_c = 2\costIDreassign.
\end{align*}

\paragraph{Case 1: $\cA=\emptyset$.}
For any $k\in[K]$, by the definition of the remaining budget in \eqref{eq:remain-bud},
\begin{align*}
    \rembud_k([n_1])-\rembud_k([n_2])&=\left(\alpha_k N - C_k^*([n_1])-\frac{\alpha_k}{3}n_1\right) - \left(\alpha_k N - C_k^*([n_2])-\frac{\alpha_k}{3}n_2\right)\\
    &= C_k^*([n_2])- C_k^*([n_1]) + \frac{\alpha_k}{3}(n_2-n_1)\\
    &\ge \frac{\alpha_k}{3}(n_2-n_1)\\
    &\ge \eta_c(n_2-n_1)-M_c.
\end{align*}

\paragraph{Case 2: $\cA\neq \emptyset$.}
In this case, for any $k\notin \cA$, following the same arguments as those in the previous paragraph, we again get $\rembud_k([n_1])-\rembud_k([n_2])\ge \eta_c(n_2-n_1)-M_c.$

Now consider any $k\in\cA$.
Let $\cD_k=\{i\in[N]\colon C_{k,i}^*\ge \costIDreassign\}.$
We first show that
\begin{equation}\label{eq:sizeDk}
\abs{\cD_k}\geq \frac{(\alpha_k/2-\costIDreassign)N}{c_{\max}-\costIDreassign}.
\end{equation}
Note that 
\begin{equation}
    \sum_{i\in [N]} C_{k,i}^* \leq c_{\max}|\cD_k|+\costIDreassign (N-|\cD_k|).
\end{equation}
Also since the type-$k$ constraint is active,
\begin{equation}
    \sum_{i\in[N]} C_{k,i}^* \ge \frac{\alpha_k}{2}N.
\end{equation}
Combining these two inequalities and recalling that $\costIDreassign = \alpha_{\min}/4 < \alpha_k/2$, where $\alpha_k/2 \le c_{\max}$ since $\alpha_kN/2\le \sum_{i\in[N]} C_{k,i}^*\le c_{\max}N$,
gives \eqref{eq:sizeDk}.

We next argue that $\cD_k$ contains enough arms to ensure the ID reassignment steps from line~\ref{alg:line:reassign-start} to line~\ref{alg:line:reassign-end} in Algorithm~\ref{alg:id-assign} can be performed.
Observe that for each $\ell = 0,1,\dots,\lfloor N/ \gsizeIDreassign\rfloor -1$, these steps remove at most $K$ elements from each $\cD_k$.
To confirm $\cD_k$ contains enough arms, note that
\begin{align*}
    \abs{\cD_k}&\geq \frac{(\alpha_k/2-\costIDreassign)N}{c_{\max}-\costIDreassign}\\
    &\ge KN\cdot \frac{\alpha_{\min}/2-\costIDreassign}{(c_{\max}-\costIDreassign)K}\\
    &\ge KN\cdot\frac{1}{\gsizeIDreassign}\\
    &\ge K\lfloor N/\gsizeIDreassign\rfloor,
\end{align*}
where we used the definition $\gsizeIDreassign = \left\lceil \frac{(c_{\max}-\costIDreassign)K}{\alpha_{\min}/2-\costIDreassign}\right\rceil$.

We are now ready to prove the inequality $\rembud_k([n_1])-\rembud_k([n_2])\ge \eta_c(n_2-n_1)-M_c$ for any $n_1,n_2$ with $1\le n_1\le n_2\le N$.
Consider the arms with new IDs in $[n_1:n_2]$.
Let $g$ be the number of groups from groups of the form $\cI(\ell)=[\ell \gsizeIDreassign+1:(\ell+1)\gsizeIDreassign]$ with $\ell=0,1,\dots,\lfloor N/\gsizeIDreassign\rfloor -1$ that are completely contained in $[n_1:n_2]$.
Then it is easy to see
\begin{equation*}
    g \ge \frac{n_2-n_1}{\gsizeIDreassign}-2.
\end{equation*}
Since Algorithm~\ref{alg:id-assign} ensures that $\sum_{i\in[N]}C_{k,i}^*\mathbbm{1}\{\newID(i)\in \cI(\ell)\}\ge \costIDreassign$ for each $\ell$, we know that
\begin{equation*}
    C_k^*([n_1:n_2])\ge \left(\frac{n_2-n_1}{\gsizeIDreassign}-2\right)\costIDreassign.
\end{equation*}
Therefore,
\begin{align*}
    \rembud_k([n_1])-\rembud_k([n_2]) & = \alpha_k N-C_k^*([n_1])-(\alpha_k N-C_k^*([n_2]))\\
    &=C_k^*([n_2])-C_k^*([n_1])\\
    &=C_k^*([n_1:n_2])\\
    &\ge \left(\frac{n_2-n_1}{\gsizeIDreassign}-2\right)\costIDreassign\\
    &\ge \eta_c(n_2-n_1)-M_c.
\end{align*}

For $\rembud(D)=\min_{k\in[K]}\rembud_k(D)$, it is straightforward to verify that
\begin{align*}
    \rembud([n_1])-\rembud([n_2]) &= \min_{k\in[K]}\rembud_k([n_1]) - \min_{k\in[K]}\rembud_k([n_2])\\
    &\ge \min_{k\in[K]}\left(\rembud_k([n_1]) - \rembud_k([n_2])\right)\\
    &\ge \eta_c(n_2-n_1)-M_c,
\end{align*}
which completes the proof.
\end{proof}

\section{Proof of main result (Theorem~\ref{thm:opt-gap-bound})}
\label{sec:proof-main}

As outlined in the technical overview in Section~\ref{sec:result-tech-overview}, the core of our proof is the construction of a Lyapunov function.
The Lyapunov function we construct is the following
\begin{equation}
    \label{eq:main-lyapunov-def}
    V(\mx) = h_{\ID}(\mx,m(\mx))+L_hN\cdot(1-m(\mx)).
\end{equation}
In the rest of this section, we first define the functions $h_{\ID}(\cdot,\cdot)$ and $m(\cdot)$, along with the constant $L_h$.
We then proceed to analyze the Lyapunov function $V$ to upper bound on the optimality gap.

\paragraph{Defining $h_{\ID}(\cdot,\cdot)$ using subset Lyapunov functions.}
We first construct a Lyapunov function indexed by a subset of arms $D\subseteq [N]$, denoted as $h(\mx,D)$, which is viewed as a function of the system state $\mx$, and it is referred to as a \emph{subset Lyapunov function}.

For each cost type $k\in[K]$, let $c_{k,i}^*(s)=\sum_{a\in\aspa} \bar{\pi}_i^*(a|s)c_{k,i}(s,a)$,
and let $c_{k}^*=(c_{k,i}^*)_{i\in[N]}$ denote the vector of the functions $c_{k,i}^*$'s.
In addition, let $r_i^*(s)=\sum_{a\in\aspa}\bar{\pi}_i^*(a|s)r_i(s,a)$,
and let $r^*=(r_i^*)_{i\in[N]}$ denote the vector of the functions $r_i^*$'s.  We combine these vectors into a set $\mathcal{G}=\{c_1^*,c_2^*,\dots,c_K^*,r^*\}$.

The subset Lyapunov function is then defined as
\begin{equation}\label{eq:h-def}
    h(\mx,D) = \max_{g\in \mathcal{G}} \sup_{\ell\in\mathbb{N}} \abs{\sum_{i\in D} \innerproduct{(x_i-\mu_i^*) P_i^{\ell} / \gamma^{\ell}}{g_i}}.
\end{equation}
Here recall that $\mx$ is an $N\times\cardS$ matrix whose $i$th row, $x_i$, describes the state of arm~$i$; $P_i$ is the transition probability matrix for arm~$i$ under the optimal single-armed policy; $\gamma = \exp(-1/(2\tau))$ for $\tau$ defined in \Cref{ass:unichain}.
To build intuition for $h(\mx,D)$, consider the term corresponding to $g=r^*$ and $\ell = 0$.
In this case, the term $\sum_{i\in D}\innerproduct{x_i-\mu_i^*}{g_i}$ measures the difference between the reward obtained by applying the optimal single-armed policies to arms in $D$ and the reward upper bound given by the LP relaxation.
A similar interpretation holds for the differences in costs.

In Lemma~\ref{lem:drift} below, we show that $h(\mx,D)$ is well-defined and establish its two key properties, which play a critical role in our analysis.  The proof of \Cref{lem:drift} is given in \Cref{app:proof-lem-drift}.
\begin{restatable}{lemma}{hproperties}\label{lem:drift}
The Lyapunov function $h(\mx, D)$ defined in \eqref{eq:h-def} is finite for all system state $\mx$ and subset $D\subseteq[N]$. 
Moreover, $h(\mx, D)$ has the following properties.
\begin{enumerate}
    \item \textbf{(Lipschitz continuity)} There exists a Lipschitz constant $L_h$ such that for each system state $\mx$ and $D'\subseteq D\subseteq [N]$, we have
    \begin{equation}
        \label{eq:h-lipschitz}
        \abs{h(\mx,D)-h(\mx,D')}\leq L_h\abs{D/D'}\,.
    \end{equation}
    \item \textbf{(Drift condition)} If each arm in $D$ takes the action sampled from the optimal single-armed policy, i.e., ${A}_{i,t}\sim\bar{\pi}_i^*(\cdot\mid S_{i,t})$, then there exists a constant $C_h > 0$ such that
    \begin{equation}
        \label{eq:h-drift-condition}
        \EE\mbracket{\bracket{h(\mX_{t+1},D)-\gamma h(\mX_t,D)}^+\middle|\mX_t,{A}_{i,t}\sim\bar{\pi}_i^*(\cdot\mid S_{i,t}),\forall i\in D}\leq C_h\sqrt{N}.
    \end{equation}
\end{enumerate}
\end{restatable}
Note that \eqref{eq:h-drift-condition} implies the following more typical form of drift condition
\begin{equation}
    \EE\mbracket{h(\mX_{t+1},D)\middle |\mX_t,{A}_{i,t}\sim\bar{\pi}_i^*(S_{i,t}),\forall i\in D}-h(\mX_{t},D)\leq -(1-\gamma)h(\mX_{t},D)+C_h\sqrt{N}.
\end{equation}

We are now ready to define the function $h_{\ID}(\cdot,\cdot)$ used to construct the Lyapunov function $V$.
For any system state $\mx$ and $m\in[0,1]_N$ (where recall that $[0,1]_N$ is the set of integer multiples of $1/N$ within the interval $[0,1]$), $h_{\ID}(\mx,m)$ is defined as
\begin{equation}
    \label{eq:hid-def}
    h_{\ID}(\mx,m)=\maxmp h(\mx,[Nm']).
\end{equation}
That is, $h_{\ID}(\mx,m)$ is an upper envelope of the subset Lyapunov functions $h(\mx,[Nm'])$'s.
The function $h_{\ID}(\mx,m)$ has properties similar to those in Lemma~\ref{lem:drift}, which we state as Lemma~\ref{lem:hID drift} and prove in Appendix~\ref{sec:properties-hID}.

\paragraph{Focus set.}
We next introduce the concept of the focus set, which is directly tied to the function $m(\cdot)$ in the Lyapunov function $V$.
The focus set is a dynamic subset of arms based on the current system state.
Specifically, for any system state $\mx$, the focus set is defined as the set $[Nm(\mx)]$, where $m(\mx)$ is given by
\begin{align}
    m(\mx)= \max \sets{ m\in[0,1]_N: h_{\ID}(\mx,m)\leq \min_{k\in[K]}\rembud_k([Nm])}.\label{eq:focus set def}
\end{align}

The focus set is introduced because it has several desirable properties that are useful for the analysis.
First, under the ID policy, almost all the arms in the focus set, except for $O(\sqrt{N})$ arms, can follow the optimal single-armed policies.
Additionally, as the focus set evolves over time, it is almost non-shrinking, and its size is closely related to the value of the function $h_{\ID}(\cdot,\cdot)$.
These properties are formalized as Lemmas~\ref{lem:majority_conformal}, \ref{lem:non-shrinking} and \ref{lem:coverage}, which are presented in Appendix~\ref{sec:lem-proof-focus-set}.

\paragraph{Bounding the optimality gap via analyzing the Lyapunov function $V$.}

With the Lyapunov function $V$ fully defined, we now proceed to bound the optimality gap $R^*(N,\bm{S}_0)-R(\pi,\bm{S}_0)$, where the policy $\pi$ is the ID policy.
As outlined in the technical overview in Section~\ref{sec:result-tech-overview}, an upper bound on the optimality gap is established via the following two lemmas.

\begin{restatable}{lemma}{optgapbyV}\label{lem:optimality gap}
Consider any $N$-armed WCMDP with initial system state $\bm{S}_0$ and assume that it satisfies Assumption \ref{ass:unichain}. Let policy $\pi$ be the ID policy.
Consider the Lyapunov function $V$ defined in \eqref{eq:main-lyapunov-def}.
Then the optimality gap of $\pi$ is bounded as
\begin{align*}
    R^*(N,\bm{S}_0)-R(\pi,\bm{S}_0)\leq \frac{2r_{\max}+L_h}{L_h N}\lim_{T\to\infty}\frac{1}{T}\sum_{t=0}^{T-1}\EE\mbracket{V(\mX_{t})} + \frac{2 r_{\max} K_{\conf}}{\sqrt{N}}\,,
\end{align*}
where $L_h$ is the Lipschitz constant in Lemma~\ref{lem:drift} and $K_{\conf}$ is the positive constant in \Cref{lem:majority_conformal}. 
\end{restatable}

\begin{restatable}{lemma}{driftV}\label{lem:drift-V}
Consider any $N$-armed WCMDP with initial system state $\bm{S}_0$ and assume that it satisfies Assumption \ref{ass:unichain}.
Let $\mX_t$ be the system state at time $t$ under the ID policy.
Consider the Lyapunov function $V$ defined in \eqref{eq:main-lyapunov-def}.
Then its drift satisfies
\begin{equation}
    \label{eq:main-drift-inequality}
    \EE\mbracket{V(\mX_{t+1}) \givenplain \mX_t}-V(\mX_t) \leq -\rhov V(\mX_t) + \Kv \sqrt{N}, 
\end{equation}
which further implies that
\begin{equation}
    \label{eq:main-final-EV-bound}
    \lim_{T\to\infty} \frac{1}{T} \sum_{t=0}^{T-1} \E{V(\mX_{t})} \leq \frac{\Kv\sqrt{N}}{\rhov},
\end{equation}
where $\rhov$ and $\Kv$ are constants whose values are given in the proof. 
\end{restatable}

The proofs of Lemmas~\ref{lem:optimality gap} and \ref{lem:drift-V} are provided in Appendix~\ref{sec:proof-lem-main-theorem}.
It is then straightforward to combine these two lemmas to get Theorem~\ref{thm:opt-gap-bound}.

\section{Lemmas and proofs for subset Lyapunov functions}
\label{sec:lem-proof-subset}
In this section, we prove lemmas for the subset Lyapunov functions $h(\mx,D)$ and $h_{\ID}(\mx,m)$, defined in \eqref{eq:h-def} and \eqref{eq:hid-def}. 
In Section~\ref{sec:preliminary-transition}, we prove a preliminary lemma that helps us utilize \Cref{ass:unichain}. 
In Section~\ref{app:proof-lem-drift}, we prove Lemma~\ref{lem:drift}, which addresses the properties of the subset Lyapunov functions $h(\mx,D)$.
Finally, in Section~\ref{sec:properties-hID}, we present and prove Lemma~\ref{lem:hID drift}, which establishes properties of the function $h_{\ID}(\mx,m)$.

\subsection{Lemma for utilizing mixing time bound of transition matrices}
\label{sec:preliminary-transition}
Recall that in \Cref{ass:unichain}, we assume that for each arm $i$, the mixing times of the transition matrix $P_i$ under the optimal single-armed policy $\pibs$ is uniformly bounded by a constant $\tau$. 
In this section, we prove a lemma that allows us to utilize this assumption. 

Let $\Xi_i$ denote a matrix whose rows are the same vector $\mu_i^*$, where recall that $\mu_i^*$ is the stationary distribution of $P_i$. 
\Cref{lem:cgamma} stated below bounds the infinite series $\sum_{\ell=0}^{\infty} \max_{i\in[N]}\norm{(P_i-\Xi_i)^\ell}_\infty \gamma^{-\ell}$ by a constant independent of $N$.

\begin{lemma}\label{lem:cgamma}
    Suppose \Cref{ass:unichain} holds, and let
    \[
        \gamma = \exp\Big(-\frac{1}{2\tau}\Big) \qquad \Ctau = \frac{4e}{1-1/\sqrt{e}} \tau.
    \]
    Then we have
    \begin{align*}
        \sum_{\ell=0}^{\infty} \max_{i\in[N]}\frac{\norm{(P_i-\Xi_i)^\ell}_\infty}{\gamma^\ell} \leq \Ctau. 
\end{align*}
 Here, recall that $\norm{A}_\infty \triangleq \max_{s\in\sspa} \sum_{s'\in\sspa}\abs{A(s,s')}$ for any matrix $A\in \R^{\sspa\times\sspa}$. 
\end{lemma}

Note that \Cref{lem:cgamma} is the only place in our proofs where \Cref{ass:unichain} is directly invoked; all of our later proofs utilize \Cref{ass:unichain} indirectly by invoking \Cref{lem:cgamma}.

\begin{proof}[Proof of \Cref{lem:cgamma}]
    For each $\ell \in \mathbb{N}$, let $\ell_0 = \ell - \tau \lfloor \ell/\tau\rfloor$. 
    By properties of the operator norm $\norm{\cdot}_\infty$, 
    \begin{equation}
        \label{eq:pf-cgamma-bound:int-divide}
        \norm{(P_i - \Xi_i)^\ell}_\infty \leq \norm{(P_i - \Xi_i)^{\ell_0}}_\infty \norm{(P_i - \Xi_i)^\tau}_\infty^{\lfloor \ell/\tau \rfloor}. 
    \end{equation}
    Because each row of $\Xi_i$ is $\mu_i^*$, it is easy to verify that $(P_i - \Xi_i)^{\ell'} = P_i^{\ell'} - \Xi_i$ for any $\ell' \in \mathbb{N}$. Consequently, 
    \begin{align*}
        \norm{(P_i - \Xi_i)^{\ell_0}}_\infty
        &= \normbig{P_i^{\ell_0} - \Xi_i}_\infty \\
        &= \max_{s\in\sspa} \sum_{s'\in\sspa} \abs{P_i^{\ell_0}(s,s') - \mu_i^*(s')}  \\
        &\leq \max_{s\in\sspa} \sum_{s'\in\sspa} \big(P_i^{\ell_0}(s,s') + \mu_i^*(s')\big) \\
        &= 2.
    \end{align*}
    Moreover, by \Cref{ass:unichain}, it is not hard to see that 
    \[
        \max_{s\in\sspa} \norm{P_i^\tau(s,\cdot) - \mu_i^*(\cdot)}_1 \leq \frac{1}{e},
    \]
    so we have
    \begin{align}
        \nonumber
        \norm{(P_i - \Xi_i)^\tau}_\infty 
        &= \max_{s\in\sspa} \sum_{s'\in\sspa} \abs{P_i^\tau(s,s') - \mu_i^*(s')} \\
        \nonumber
        &= \max_{s\in\sspa} \norm{P_i^\tau(s,\cdot) - \mu_i^*(\cdot)}_1 \\
        \nonumber
        &\leq \frac{1}{e}.
    \end{align}
    Substituting the bounds on $\norm{(P_i - \Xi_i)^\tau}_\infty$ and $\norm{(P_i - \Xi_i)^{\ell_0}}_\infty$ back to \eqref{eq:pf-cgamma-bound:int-divide}, we get 
    \begin{align}
        \nonumber
        \norm{(P_i - \Xi_i)^\ell}_\infty
        &\leq 2\exp\Big(-\Big\lfloor \frac{\ell}{\tau} \Big\rfloor\Big) \\
        &\leq 2\exp\Big(-\frac{\ell}{\tau} + 1\Big).
    \end{align}

    Therefore, 
    \begin{align}
        \nonumber
        \sum_{\ell=0}^{\infty} \max_{i\in[N]}\frac{\norm{(P_i-\Xi_i)^\ell}_\infty}{\gamma^\ell}
        &\leq \sum_{\ell=0}^{\infty} 2\exp\Big(-\frac{\ell}{\tau} + 1\Big) \gamma^{-\ell} \\
        \nonumber
        &= \sum_{\ell=0}^{\infty} 2\exp\Big(-\frac{\ell}{2\tau} + 1\Big) \\
        &= \frac{2e}{1 - \exp(-1/(2\tau))}. 
    \end{align}
    Observe that $\tau \geq 1$, $0 < 1/(2\tau) \leq 1/2$, and $\exp(-a) \leq 1 - (1-1/\sqrt{e}) a$ for any $0 < a \leq 1/2$, so
    \[
        \frac{2e}{1 - \exp(-1/(2\tau))} \leq \frac{4e}{1-1/\sqrt{e}} \tau = \Ctau. \qedhere
    \]
\end{proof}

\subsection{Proof of Lemma~\ref{lem:drift}}
\label{app:proof-lem-drift}

In this subsection, we prove \Cref{lem:drift}, which is about properties of the Lyapunov function $h(\mx,D)$:
\begin{equation}\tag{\ref{eq:h-def}}
    h(\mx,D) \triangleq \max_{g\in \mathcal{G}} \sup_{\ell\in\mathbb{N}} \abs{\sum_{i\in D} \innerproduct{(x_i-\mu_i^*) P_i^{\ell} / \gamma^{\ell}}{g_i}}.
\end{equation}
\Cref{lem:drift} is restated as follows. 
\hproperties*

In the proof of \Cref{lem:drift}, we will frequently use the following equivalent form of $h(\mX, D)$:  
\begin{equation}\label{eq:h-equiv-form}
    h(\mx,D) = \max_{g\in \mathcal{G}} \sup_{\ell\in\mathbb{N}} \abs{\sum_{i\in D} \innerproduct{(x_i-\mu_i^*) (P_i - \Xi_i)^{\ell} / \gamma^{\ell}}{g_i}},
\end{equation}
where $\Xi_i$ is the matrix whose each row is the optimal state distribution of the $i$-th arm, $\mu^*_i$. 
The equation \eqref{eq:h-equiv-form} is equivalent to \eqref{eq:h-def} because $(v_1 - v_2) P_i^\ell = (v_1-v_2) (P_i - \Xi_i)^\ell$ for any $i\in [N]$, $\ell \in \mathbb{N}$, and row vectors $v_1, v_2 \in \Delta(\sspa)$.

\begin{proof}
We organize the proof into three parts: we first show the finiteness of the subset Lyapunov function $h(\mx, D)$; then, we prove the Lipschitz continuity of $h(\mx, D)$ with respect to $D$ \eqref{eq:h-lipschitz}; finally, we prove the drift condition for $h(\mx, D)$ stated in \eqref{eq:h-drift-condition}. 

\paragraph{Finiteness of $h(\mx, D)$.}
To show that $h(\mx,D)$ is finite for any system state $\mx$ and subset $D\subseteq [N]$, we have for any $g\in \mathcal{G}$: 
\begin{align}
    \nonumber
    \supell \abs{\sum_{i\in D} \innerproduct{(x_i-\mu_i^*) P_i^{\ell} / \gamma^{\ell}}{g_i}}
    &\leq \sum_{i\in D} \sumell \abs{\innerproduct{(x_i-\mu_i^*) P_i^{\ell} / \gamma^{\ell}}{g_i}}\\
    \label{eq:pf-subset-lyapunov:finite-1}
    &= \sum_{i\in D}\sumell \abs{\innerproduct{(x_i-\mu_i^*) (P_i-\Xi_i)^{\ell} / \gamma^{\ell}}{g_i}} \\
    \label{eq:pf-subset-lyapunov:finite-2}
    &\leq \sum_{i\in D} \sumell \norm{x_i-\mu_i^*}_1\frac{\norm{(P_i-\Xi_i)^{\ell}}_\infty}{\gamma^{\ell}}\norm{g_i}_\infty,
\end{align}
By \Cref{lem:cgamma}, $\sumell \norm{(P_i-\Xi_i)^{\ell}}_\infty / \gamma^{\ell}$ is finite, so the expression in \eqref{eq:pf-subset-lyapunov:finite-2} is also finite. Taking maximum over $g\in \mathcal{G}$, because $\mathcal{G}$ is a finite set, we have 
\[
    h(\mx, D) = \max_{g\in \mathcal{G}} \supell \abs{\sum_{i\in D} \innerproduct{(x_i-\mu_i^*) P_i^{\ell} / \gamma^{\ell}}{g_i}} < \infty.
\]

\paragraph{Lipschitz continuity.}
For any system state $\mx$ and subsets $D,D'$ such that $D'\subseteq D \subseteq [N]$, we have
\begin{align}
    \nonumber
    \abs{h(\mx,D)-h(\mx,D')}
    &=  \abs{\max_{g\in \mathcal{G}}  \supell 
 \abs{\sum_{i\in D} \innerproduct{(x_i-\mu_i^*) P_i^{\ell} / \gamma^{\ell}}{g_i}} -  
    \max_{g\in \mathcal{G}}  \supell \abs{\sum_{i\in D'} \innerproduct{(x_i-\mu_i^*) P_i^{\ell} / \gamma^{\ell}}{g_i}}} \\
    \label{eq:pf-subset-lyapunov:lipschitz-1}
    &\leq \max_{g\in \mathcal{G}} \supell \abs{\sum_{i\in D/D'}\innerproduct{(x_i-\mu_i^*)P_i^{\ell}/\gamma^{\ell}}{g_i}}. 
\end{align}
Following similar arguments used to proving finiteness of $h(\mx, D)$, we further bound the last expression as: 
\begin{align}
    \nonumber
    \max_{g\in \mathcal{G}} \supell \abs{\sum_{i\in D/D'}\innerproduct{(x_i-\mu_i^*)P_i^{\ell}/\gamma^{\ell}}{g_i}} 
    &\leq  \max_{g\in \mathcal{G}}  \sum_{i\in D/D'} \norm{x_i-\mu_i^*}_1 \norm{g_i}_\infty \sumell \frac{\norm{(P_i-\Xi_i)^{\ell}}_\infty}{\gamma^{\ell}}\\
    \label{eq:pf-subset-lyapunov:lipschitz-2}
    &\leq 2\abs{D/D'} \max\sets{c_{\max},r_{\max}}\Ctau,
\end{align}
where in the last inequality, we utilize the facts that for any $g\in \mathcal{G}$ and $i\in[N]$, $\norm{g_i}_\infty \leq \max\sets{c_{\max},r_{\max}}$,  $\norm{x_i - \mu_i^*}_1 \leq 2$, and $\sumell \norm{(P_i-\Xi_i)^{\ell}}_\infty / \gamma^{\ell} \leq \Ctau$. 
Therefore, $h(\mx, D)$ is Lipschitz continuous in $D$ with the Lipschitz constant 
$L_h = 2\max\sets{c_{\max},r_{\max}}\Ctau$.

\paragraph{Drift condition.}
Next, we prove the drift condition in \eqref{eq:h-drift-condition}, which requires showing 
\[
    \E{(h(\mX_{t+1}, D) - \gamma h(\mX_t, D))^+ \givenplain \mX_t} = O(1/\sqrt{N}),
\]
when the $i$-th arm follows the action generated by $\pibs_i$ for each $i\in D$. 
Because $D$ is fixed in the rest of the proof, for simplicity, we use $h(\mx)$ as shorthand for $h(\mx, D)$. 

We first perform the following decomposition: 
\begin{equation}
    \label{eq:pf-subset-lyapunov:drift-decompose}
    \EE\mbracket{\bracket{h(\mX_{t+1})-\gamma h(\mX_t)}^+\givenmiddle\mX_t} \leq \EE\mbracket{\abs{h(\mX_{t+1})-h(\mX_tP)}\givenmiddle\mX_t} + \bracket{h(\mX_tP)-\gamma h(\mX_t)}^+,
\end{equation}
where $\mX_t P \in \R^{N\times \cardS}$ denotes the matrix whose $i$-th row is given by $(\mX_t P)_i \triangleq X_{i,t}P_i$, which is the state distribution of arm $i$ after one-step of transition from $X_{i,t}$ under the policy $\pibs_i$. 
Next, we bound the two terms on the right-hand side of \eqref{eq:pf-subset-lyapunov:drift-decompose} separately.

We first bound the term $\bracket{h(\mX_tP)-\gamma h(\mX_t)}^+$. Substituting $\mX_t P$ into the definition of $h$, we have
\begin{align}
    \nonumber
    h(\mX_tP) &= \max_{g\in \mathcal{G}}  \supell \abs{\sum_{i\in D} \innerproduct{(X_{i,t}-\mu_i^*)P_i^{\ell+1}/\gamma^{\ell}}{g_i}}\\
    \nonumber
    &=\gamma \max_{g\in \mathcal{G}}  \sup_{\ell\in\mathbb{N}\colon \ell \geq 1}  \abs{\sum_{i\in D} \innerproduct{(X_{i,t}-\mu_i^*)P_i^{\ell}/\gamma^{\ell}}{g_i}}\\
    \label{eq:h-contraction}
    &\leq \gamma h(\mX_t),
\end{align}
so $\bracket{h(\mX_tP)-\gamma h(\mX_t)}^+ = 0$ with probability $1$. 

Next, we bound the term $\EE\mbracket{\abs{h(\mX_{t+1})-h(\mX_tP)}\givenmiddle\mX_t}$. 
Let $\epsilon_{i,t} \in \R^{\sspa}$ be the random vector given by $\epsilon_{i,t} = X_{i,t+1}-X_{i,t}P_i$ for $i\in D$. Then for each state $s\in\sspa$, $\epsilon_{i,t}(s)$ conditioned on $\mX_t$ is a Bernoulli distribution with mean $0$. Consequently, 
\begin{align}
    \nonumber
    &\mspace{23mu} \abs{h(\mX_{t+1})-h(\mX_tP)}\\
    \nonumber
    &\leq \max_{g\in \mathcal{G}} \supell \abs{\absbigg{\sum_{i\in D}\innerproduct{(X_{i,t+1}-\mu_i^*)P_i^{\ell}/\gamma^{\ell}}{g_i}}-\absbigg{\sum_{i\in D} \innerproduct{(X_{i,t}P_i-\mu_i^*)P_i^{\ell}/\gamma^{\ell}}{g_i}}}\\ 
    \nonumber 
    &\leq \max_{g\in \mathcal{G}} \supell \abs{\sum_{i\in D}\innerproduct{\epsilon_{i,t}P_i^{\ell}/\gamma^{\ell}}{g_i}} \\
    \nonumber
    &= \max_{g\in \mathcal{G}} \supell \abs{\sum_{i\in D}\innerproduct{\epsilon_{i,t}(P_i-\Xi_i)^{\ell}/\gamma^{\ell}}{g_i}}\\
    \label{eq:pf-subset-lyapunov:drift-1}
    &= \max_{g\in \mathcal{G}} \supell \gamma^{-\ell} \abs{\sum_{i\in D}\innerproduct{\epsilon_{i,t}}{(P_i-\Xi_i)^{\ell}g_i}}. 
\end{align}
Taking the conditional expectation of $\abs{h(\mX_{t+1})-h(\mX_tP)}$ given $\mX_t$, then \eqref{eq:pf-subset-lyapunov:drift-1} implies that 
\begin{align}
    \nonumber \EE\mbracket{\abs{h(\mX_{t+1})-h(\mX_tP)}\givenmiddle\mX_t}
    &\leq \E{\max_{g\in \mathcal{G}} \supell \gamma^{-\ell} \abs{\sum_{i\in D}\innerproduct{\epsilon_{i,t}}{(P_i-\Xi_i)^{\ell}g_i}} \givenmiddle \mX_t} \\
     \label{eq:pf-subset-lyapunov:drift-1-5}
    &\leq \sum_{g\in \mathcal{G}} \sumell  \gamma^{-\ell} \E{ \abs{\sum_{i\in D}\innerproduct{\epsilon_{i,t}}{(P_i-\Xi_i)^{\ell}g_i}} \givenmiddle \mX_t}.
\end{align}

We bound the conditional expectation in \eqref{eq:pf-subset-lyapunov:drift-1-5} as follows: for any $\ell \in \mathbb{N}$ and $g\in \mathcal{G}$, we have
\begin{align}
    \E{\abs{\sum_{i\in D}\innerproduct{\epsilon_{i,t}}{(P_i-\Xi_i)^{\ell}g_i}} \givenmiddle \mX_t}
    \nonumber
    &\leq \sqrt{\E{\bigg(\sum_{i\in D}\innerproduct{\epsilon_{i,t}}{(P_i-\Xi_i)^{\ell}g_i}\bigg)^2 \givenmiddle \mX_t}}  \\
    \label{eq:pf-subset-lyapunov:drift-2}
    &= \sqrt{\sum_{i\in D}\E{\innerproduct{\epsilon_{i,t}}{(P_i-\Xi_i)^{\ell}g_i}^2 \givenmiddle \mX_t}} \\
    \nonumber
    &\leq  \sqrt{\sum_{i\in D}\E{\Big(\norm{\epsilon_{i,t}}_1 \norm{(P_i-\Xi_i)^{\ell}}_\infty \norm{g_i}_\infty\Big)^2 \givenmiddle \mX_t}} \\
    \label{eq:pf-subset-lyapunov:drift-3}
    &\leq 2 \sqrt{\abs{D}} \max_{i\in D} \norm{g_i}_{\infty}\norm{(P_i-\Xi_i)^{\ell}}_{\infty},
\end{align}
where \eqref{eq:pf-subset-lyapunov:drift-2} uses the facts that $\epsilon_{i,t}$ are independent across $i\in D$ and $\innerproduct{\epsilon_{i,t}}{(P_i-\Xi_i)^{\ell}g_i}$ has zero mean conditioned on $\mX_t$; \eqref{eq:pf-subset-lyapunov:drift-3} is because $\norm{\epsilon_{i,t}}_1 = \norm{X_{i,t+1}-X_{i,t}P_i}_1 \leq 2$ for each $i\in D$. 

Substituting the bound in \eqref{eq:pf-subset-lyapunov:drift-3} back to  \eqref{eq:pf-subset-lyapunov:drift-1-5}, we get 
\begin{align}
    \nonumber \EE\mbracket{\abs{h(\mX_{t+1})-h(\mX_tP)}\givenmiddle\mX_t}
    &\leq 2\sqrt{\abs{D}} \sumell \sum_{g\in \mathcal{G}}  \gamma^{-\ell} \max_{i\in D}  \norm{g_i}_{\infty} \norm{(P_i-\Xi_i)^{\ell}}_{\infty} \\
    \nonumber
    &= 2\sqrt{\abs{D}} \sumell \sum_{g\in \mathcal{G}}   \max_{i\in D}  \norm{g_i}_{\infty} \cdot \max_{i\in D} \gamma^{-\ell}\norm{(P_i-\Xi_i)^{\ell}}_{\infty} \\
    \nonumber
    &\leq 2\sqrt{\abs{D}}  (Kc_{\max} + r_{\max}) \sumell \max_{i\in D}\frac{\norm{(P_i-\Xi_i)^{\ell}}_{\infty}}{\gamma^{\ell}} \\
     \label{eq:pf-subset-lyapunov:drift-5}
    &\leq 2\sqrt{N} (Kc_{\max} + r_{\max}) \Ctau,
\end{align}
where the inequality in \eqref{eq:pf-subset-lyapunov:drift-5} is due to $\sumell \max_{i\in D} \norm{(P_i-\Xi_i)^{\ell}}_{\infty} / \gamma^{\ell} \leq \Ctau$ (\Cref{lem:cgamma}).

Combining the above calculations, we get:
\begin{align*}
    \EE\mbracket{\bracket{h(\mX_{t+1})-\gamma h(\mX_t)}^+\givenmiddle\mX_t} &\leq \EE\mbracket{\abs{h(\mX_{t+1})-h(\mX_tP)}\givenmiddle\mX_t}+ \bracket{h(\mX_tP)-\gamma h(\mX_t)}^+\\
    &\leq 2\sqrt{N} (Kc_{\max} + r_{\max}) \Ctau. 
\end{align*}
Therefore, $\EE\mbracket{\bracket{h(\mX_{t+1})-\gamma h(\mX_t)}^+\givenmiddle\mX_t}\leq C_h \sqrt{N}$ with $C_h = 2(Kc_{\max} + r_{\max}) \Ctau$. 
\end{proof}

\subsection{Properties of $h_{\ID}(\cdot,\cdot)$}
\label{sec:properties-hID}

\begin{lemma}\label{lem:hID drift}
The Lyapunov function $\hid(\mx, m)$ defined in \eqref{eq:hid-def} has the following properties: 
\begin{enumerate}
    \item \textbf{(Lipschitz continuity)} For each system state $\mx$ and $m, m'\in[0,1]_N$, we have
    \begin{equation}
        \label{eq:hid-drift-condition}
        \abs{h_{\ID}(\mx,m)-h_{\ID}(\mx,m')}\leq NL_h\abs{m-m'}\,,
    \end{equation}
    where $L_h > 0$ is the Lipschitz constant given in \Cref{lem:drift}. 
    \item \textbf{(Drift condition)} 
        For each $m\in[0,1]_N$, if all arms in $[Nm]$ follow the optimal single-armed policies, we have: 
        \begin{align*}
            \EE\mbracket{(h_{\ID}(\mX_{t+1},m)-\gamma h_{\ID}(\mX_t,m))^+\mid\mX_t, A_{i,t}\sim\Bar{\pi}_i^*(\cdot\mid S_{i,t}),\forall i \in [Nm]}\leq 2C_h\sqrt{N},
        \end{align*}
        where $C_h > 0$ is the constant given in \Cref{lem:drift}.  
\end{enumerate}
\end{lemma}

\begin{proof}\label{proof:hID drift}
    We first prove the Lipschitz continuity of $\hid(\mx, m)$ with respect to $m$. 
    Because $\hid(\mx, m)$ is non-decreasing in $m$, it suffices to demonstrate that for any $m, m'\in[0,1]_N$ such that $m > m'$, 
    \begin{equation}
        \label{pf:hid-lipschitz:goal}
        \hid(\mx, m) - \hid(\mx, m') \leq NL_h (m-m').
    \end{equation}
    Denote $m_1=\argmax_{m_1\in[0,1]_N\colon m_1\leq m}h(\mx,[Nm_1])$. Then, by the definition of $h_{\ID}$, we have $h_{\ID}(\mx,m)=h(\mx,[Nm_1])$ and 
    \begin{equation}
        \label{pf:hid-lipschitz:eq-1}
        \hid(\mx, m) - \hid(\mx, m') =  h(\mx, [Nm_1]) - \hid(\mx, m'). 
    \end{equation}
    If $m_1\leq m'$, the right-hand side of \eqref{pf:hid-lipschitz:eq-1} is non-positive, so \eqref{pf:hid-lipschitz:goal} follows. 
    If $m'< m_1 \leq m$, because $\hid(\mx, m')\geq h(\mx, [Nm'])$,  \eqref{pf:hid-lipschitz:eq-1} implies that
    \begin{align}
        \nonumber
        h_{\ID}(\mx,m)-h_{\ID}(\mx,m')
        &\leq h(\mx, [Nm_1])  - h(\mx, [Nm']) \\
        \label{eq:apply-h-lipschitz}
        &\leq L_h (Nm_1-Nm')\\
        \nonumber
        &\leq NL_h(m-m')\,,
    \end{align}
    where \eqref{eq:apply-h-lipschitz} is due to the Lipschitz continuity of $h(x, D)$ with respect to $D$, as established in \Cref{lem:drift}. 
    We have thus proved \eqref{pf:hid-lipschitz:goal}. 

    Next, we prove the drift condition. We will assume $A_{i,t}\sim \pibs_i(\cdot\givenplain S_{i,t})$ for all $i\in [Nm]$ in the rest of the proof, without explicitly writing it in the conditional probabilities each time. 
    We start by bounding the following expression: 
    \begin{align}
        \nonumber
        &\mspace{23mu} h_{\ID}(\mX_{t+1},m)-\gamma h_{\ID}(\mX_{t},m)\\
        \nonumber
        &= \maxmp h(\mX_{t+1},[Nm']) - \maxmp \gamma h(\mX_t,[Nm'])\\
        \nonumber
        &\leq \maxmp \bracket{h(\mX_{t+1},[Nm'])-\gamma h(\mX_t,[Nm'])}\\ 
        \label{eq:pf-id-drift:apply-h-contraction}
        &\leq \maxmp \bracket{h(\mX_{t+1},[Nm']) - h(\mX_tP,[Nm'])}\\
        \label{eq:pf-id-drift:eq-1}
        &\leq \maxmp \max_{g\in \mathcal{G}} \supell \gamma^{-\ell} \abs{\sum_{i\in[Nm']} \innerproduct{\epsilon_{i,t}}{(P_i-\Xi_i)^{\ell}g_i}}\,,
    \end{align}
    where \eqref{eq:pf-id-drift:apply-h-contraction} follows from \eqref{eq:h-contraction}; $\epsilon_{i,t}$ is a $\cardS$-dimensional random vector given by $\epsilon_{i,t} \triangleq X_{i,t+1}-X_{i,t}P_i$, 
    which satisfies $\EE\mbracket{\epsilon_{i,t}(s)\givenmiddle\mX_t}=0$ for any $s\in\sspa$ and $\norm{\epsilon_{i,t}}_1 \leq 2$; \eqref{eq:pf-id-drift:eq-1} follows \eqref{eq:pf-subset-lyapunov:drift-1}. 
    Now, we take the expectations of the positive parts of the inequality \eqref{eq:pf-id-drift:eq-1} conditioned on $X_t$, which yields
    \begin{align}
        \nonumber
        &\mspace{23mu} \EE\mbracket{(h_{\ID}(\mX_{t+1},m)-\gamma h_{\ID}(\mX_{t},m))^+\givenmiddle\mX_t}\\
        \nonumber
        &\leq \E{\maxmp \max_{g\in \mathcal{G}} \supell \gamma^{-\ell} \abs{\sum_{i\in[Nm']} \innerproduct{\epsilon_{i,t}}{(P_i-\Xi_i)^{\ell}g_i}} \givenmiddle \mX_t} \\
        \nonumber
        &\leq  \sumell \gamma^{-\ell} \sum_{g\in \mathcal{G}} \E{ \maxmp \abs{\sum_{i\in[Nm']} \innerproduct{\epsilon_{i,t}}{(P_i-\Xi_i)^{\ell}g_i}} \givenmiddle \mX_t} \\
        \label{eq:pf-id-drift:eq-2}
        &\leq \sumell \gamma^{-\ell} \sum_{g\in \mathcal{G}} \E{ \maxmp \abs{\sum_{i\in[Nm']} \innerproduct{\epsilon_{i,t}}{(P_i-\Xi_i)^{\ell}g_i}}^2 \givenmiddle \mX_t}^{1/2}.
    \end{align}
    Because $\epsilon_{i,t}$ are independent across $i\in [Nm]$ and   $\innerproduct{\epsilon_{i,t}}{(P_i-\Xi_i)^{\ell}g_i}$ has zero mean conditioned on $\mX_t$, the partial sum $\sum_{i\in[n]} \innerproduct{\epsilon_{i,t}}{(P_i-\Xi_i)^{\ell}g_i}$ is a martingale in $n$ and $\abs{\sum_{i\in[n]} \innerproduct{\epsilon_{i,t}}{(P_i-\Xi_i)^{\ell}g_i}}$ is a submartingale in $n$ for $n\in [Nm]$. Then it follows from Doob's $L_2$ inequality (\citealp[Theorem 5.4.3]{Dur_19_prob_book}) that
    \begin{align}
        \nonumber
        &\mspace{23mu} \E{ \maxmp \abs{\sum_{i\in[Nm']} \innerproduct{\epsilon_{i,t}}{(P_i-\Xi_i)^{\ell}g_i}}^2 \givenmiddle \mX_t}^{1/2}  \\
        \nonumber
        &\leq 2 \E{ \bigg( \sum_{i\in[Nm]} \innerproduct{\epsilon_{i,t}}{(P_i-\Xi_i)^{\ell}g_i} \bigg)^2 \givenmiddle \mX_t}^{1/2} \\
        \nonumber
        &= 2  \left( \sum_{i\in[Nm]} \E{ \innerproduct{\epsilon_{i,t}}{(P_i-\Xi_i)^{\ell}g_i}^2 \givenmiddle \mX_t}\right)^{1/2} \\
        \nonumber
        &\leq  2 \left(\sum_{i\in [Nm]}\E{\Big(\norm{\epsilon_{i,t}}_1 \norm{(P_i-\Xi_i)^{\ell}}_\infty \norm{g_i}_\infty\Big)^2 \givenmiddle \mX_t}\right)^{1/2} \\
        \label{eq:pf-id-drift:eq-3}
        &\leq 4 \sqrt{Nm} \max_{i\in [Nm]} \norm{g_i}_{\infty}\norm{(P_i-\Xi_i)^{\ell}}_{\infty},
    \end{align}
    where \eqref{eq:pf-id-drift:eq-3} uses the fact that $\norm{\epsilon_{i,t}}_1 \leq 2$. 
    Substituting the bound in \eqref{eq:pf-id-drift:eq-3} back to \eqref{eq:pf-id-drift:eq-2}, we get 
    \begin{align*}
        \EE\mbracket{(h_{\ID}(\mX_{t+1},m)-\gamma h_{\ID}(\mX_{t},m))^+\givenmiddle\mX_t}
        &\leq  4 \sqrt{Nm}  \sumell \gamma^{-\ell} \sum_{g\in \mathcal{G}}\max_{i\in [Nm]} \norm{g_i}_{\infty}\norm{(P_i-\Xi_i)^{\ell}}_{\infty} \\
        &= 4 \sqrt{N}   (K c_{\max} + r_{\max})\sumell \gamma^{-\ell}\max_{i\in [N]}\norm{(P_i-\Xi_i)^{\ell}}_{\infty} \\
        \nonumber
        &\leq 4\sqrt{N} (Kc_{\max} + r_{\max}) \Ctau \\
        &= 2C_h\sqrt{N}.  \qedhere
    \end{align*}
\end{proof}

\section{Lemmas for focus set}
\label{sec:lem-proof-focus-set}

In this section, we present and prove three lemmas about properties of the focus set.
Recall that for any system state $\mx$, the focus set is defined as the set $[Nm(\mx)]$, where $m(\mx)$ is given by
\begin{align}
    m(\mx)= \max \sets{ m\in[0,1]_N: h_{\ID}(\mx,m)\leq \min_{k\in[K]}\rembud_k([Nm])}.
\end{align}
Consider the system state process under the ID policy, $(\bm{S}_t,t\in\mathbb{N})$ and its equivalent representation $(\mX_t,t\in\mathbb{N})$.
We often consider the focus set corresponding to the current system state, i.e., $m(\mX_t)$.
A closely related quantity is the number of arms that follow their optimal single-armed policies under the ID policy, which we refer to as the \emph{conforming number}.
With the system state $\bm{S}_t$, the conforming number is denoted as $N^*_t$, and it can be written as
\begin{equation}
    N_t^* = \max\sets{n\in[N]\colon \sum_{i=1}^n c_{k,i}(S_{i,t},\widehat{A}_{i,t})\leq \alpha_k N, \forall k\in[K]},
\end{equation}
where $\widehat{A}_{i,t}$'s are the ideal actions sampled from the optimal single-armed policies by the ID policy.

Below we state the three lemmas,  and we prove them in the subsections.

\begin{lemma}[Majority conformity]\label{lem:majority_conformal}
Let $(\mX_t,t\in\mathbb{N})$ be the system state process under the ID policy.
The size of the focus set, $Nm(\mX_t)$, satisfies
\begin{align*}
    \frac{1}{N}\EE[\bracket{Nm(\mX_t) - N_t^* }^+ \givenplain \mX_t]\leq \frac{K_{\conf}}{\sqrt{N}}, \quad\text{with probability }1,
\end{align*}
for some constant $K_{\conf}>0$ independent of $N$.
\end{lemma}
Lemma~\ref{lem:majority_conformal} implies that almost all the arms in the focus set, except for $O(\sqrt{N})$ arms, can follow the optimal single-armed policies.

\begin{lemma}[Almost non-shrinking]\label{lem:non-shrinking}
Let $(\mX_t,t\in\mathbb{N})$ be the system state process under the ID policy.
Then the change in the size of the focus set over time satisfies
    \begin{align*}
    \EE\mbracket{(m(\mX_t)-m(\mX_{t+1}))^+ \givenmiddle \mX_t} \leq  \frac{K_{\mono}}{\sqrt{N}},\quad\text{with probability }1,
\end{align*}
for some constant $K_{\mono}>0$ independent of $N$.
\end{lemma}
Lemma~\ref{lem:non-shrinking} implies that the size of the focus set is almost non-shrinking on average over time, or more specifically, it shrinks by at most $O(\sqrt{N})$ on average over time.

\begin{lemma}[Sufficient coverage]\label{lem:coverage}
Let $(\mX_t,t\in\mathbb{N})$ be the system state process under the ID policy.
Then
\begin{align*}
    1-m(\mX_t)\leq \frac{1}{\eta_c N}h_{\ID}(\mX_t,m(\mX_t)) + \frac{K_{\cov}}{N},\quad\text{with probability }1,
\end{align*}
for some constant $K_{\cov}>0$ independent of $N$.
\end{lemma}
Lemma~\ref{lem:coverage} relates the size of the complement of the focus set to the value of the function $h_{\ID}(\mX_t,m(\mX_t))$.

\subsection{Proof of Lemma~\ref{lem:majority_conformal} (Majority conformity)}
\begin{proof}
First, we claim that the conforming number $N_t^*$ can be lower bounded using our slack budget function $\rembud$ and Lyapunov function as follows: for any time step $t\geq 0$, 
\begin{align}\label{eq:nt eq1}
    N_t^* \geq N \max\sets{m\in[0,1]_N\colon   \rembud([Nm])-h_{\ID}(\mX_t,m)\geq  \Delta_t}\,,
\end{align}
where $\Delta_t$ is the random variable given by 
\begin{align}\label{eq:deltat}
    \Delta_t = \maxk \max_{m\in[0,1]_N} \abs{\sum_{i\in[Nm]}\bracket{c_{k,i}(S_{i,t},\widehat{A}_{i,t})- \innerproduct{X_{i,t}}{c_{k,i}^*}}}\,,
\end{align}
which captures the difference between the expected cost and the actual cost.  

To prove the claim, we invoke the definition of $N_t^*$ and the fact that $\rembud([n])\leq \alpha_k N - C_k([n])$ for any $n\in [N]$ and $k\in[K]$: 
    \begin{align}
        \nonumber
        N_t^* 
        &= \max\sets{n\in[N]\colon  \sum_{i\in[n]} c_{k,i}(S_{i,t},\widehat{A}_{i,t})  \leq \alpha_k N \quad \forall k\in[K]}     \\
        \label{eq:pf-conformity:applied-nt-def}
        &\geq \max\sets{n\in[N]\colon \maxk  \sum_{i\in[n]} \left(c_{k,i}(S_{i,t},\widehat{A}_{i,t}) - C_k([n]) \right) \leq \rembud([n])}.
    \end{align}
    To further lower bound \eqref{eq:pf-conformity:applied-nt-def}, we identify a subset of the set in \eqref{eq:pf-conformity:applied-nt-def} and reduce the task to upper bounding the following expression: 
    \begin{align}
        \nonumber
        &\mspace{23mu} \maxk  \sum_{i\in[n]} \left( c_{k,i}(S_{i,t},\widehat{A}_{i,t}) - C_k([n]) \right) \\
        \nonumber
        &\leq\maxk \abs{\sum_{i\in[n]}\innerproduct{X_{i,t}}{c_{k,i}^*}- C_k([n])}+ \maxk \abs{\sum_{i\in[n]}c_{k,i}(S_{i,t},\widehat{A}_{i,t})- \sum_{i\in[n]}\innerproduct{X_{i,t}}{c_{k,i}^*}}\\
        \label{eq:pf-conformity:eq-1}
        &\leq h_{\ID}(\mX_t,n/N) + \maxk \abs{\sum_{i\in[n]}\bracket{c_{k,i}(S_{i,t},\widehat{A}_{i,t})- \innerproduct{X_{i,t}}{c_{k,i}^*}}} \\
        \label{eq:pf-conformity:eq-2}
        &\leq h_{\ID}(\mX_t,n/N) + \Delta_t, 
    \end{align}
    where the inequality \eqref{eq:pf-conformity:eq-1} follows from the argument that
    \begin{align*}
        h_{\ID}(\mX_t,n/N)&=\max_{n'\in[n]}\max_{g\in \mathcal{G}} \supell \abs{\sum_{i\in[n']}\innerproduct{(X_{i,t}-\mu_i^*)P_i^{\ell}/\gamma^{\ell}}{g_i}}\\
        &\geq \maxk\abs{\sum_{i\in[n]}\innerproduct{X_{i,t}-\mu_i^*}{c_{k,i}^*}}\\
        &=\maxk\abs{\sum_{i\in[n]}\innerproduct{X_{i,t}}{c_{k,i}^*}- C_k([n])},
    \end{align*} 
    and \eqref{eq:pf-conformity:eq-2} follows from the definition of $\Delta_t$ in \eqref{eq:deltat} by taking $m=n/N$ in the maximum. 
    Substituting the bound in \eqref{eq:pf-conformity:eq-2} into \eqref{eq:pf-conformity:applied-nt-def}, we get
    \[
        N^*_t \geq \max\sets{n\in[N]\colon h_{\ID}(\mX_t,n/N) + \Delta_t \leq \rembud([n])},
    \]
    which is equivalent to the claim in \eqref{eq:nt eq1}.

We now use \eqref{eq:nt eq1} to prove the lemma. 
Observe that by the definition of $m(\mX_t)$, 
\[
    \rembud([Nm(\mX_t)])-h_{\ID}(\mX_t,m(\mX_t)) \geq 0.
\]
and $h_{\ID}(\mX_t,m)$ is non-decreasing in $m$. Consequently, for any $m\in [0,1]_N$ such that $m\leq m(\mX_t)$, 
\begin{align}
    \nonumber
    \rembud([Nm])-h_{\ID}(\mX_t,m) 
    &\geq \rembud([Nm])-h_{\ID}(\mX_t,m(\mX_t)) \\
    \nonumber
    &\geq  \rembud([Nm]) - \rembud([Nm(\mX_t)]) \\\nonumber
    &=\min_{k\in[K]} \rembud_k([Nm])-\min_{k\in[K]} \rembud_k([Nm(\mX_t)])\\\nonumber
    &\geq \min_{k\in[K]} \bracket{\rembud_k([Nm]) - \rembud_k([Nm(\mX_t)])}
    \label{eq:pf-conform:apply-strict-slope}\\
    &\geq \eta_cN(m(\mX_t)-m)-M_c,
\end{align}
where \eqref{eq:pf-conform:apply-strict-slope} follows from the strict slope of the remaining budget $\rembud([Nm])$ (\Cref{lem:positiveC}). 
Therefore, choosing $m = m(\mX_t) - (\Delta_t+M_c)/(\eta_c N)$, we obtain
\begin{align*}
   \rembud([Nm])-h_{\ID}(\mX_t,m)\geq \Delta_t.
\end{align*}
Recalling the lower bound of $N_t^*$ established in \eqref{eq:nt eq1}, we arrive at
\begin{align*}
    N_t^*\geq N\max\sets{m\in[0,1]_N: \rembud([Nm])-h_{\ID}(\mX_t,m)\geq\Delta_t} \geq Nm(\mX_t)-\frac{\Delta_t +M_c}{\eta_c}\,.
\end{align*}
Rearranging the terms and taking the conditional expectations, we establish that
\begin{equation}
    \label{eq:pf-conform:bound-in-terms-of-E-delta}
    \E{\big(Nm(\mX_t) - N_t^*\big)^+ \givenmiddle \mX_t} \leq \EE\mbracket{\frac{\Delta_t+M_c}{\eta_c} \givenmiddle \mX_t}. 
\end{equation}

It remains to upper bound the conditional expectation of $\Delta_t$ given $\mX_t$. 
We define the random variable $\xi_{k,i} \triangleq c_{k,i}(S_{i,t},\widehat{A}_{i,t})- \langle X_{i,t}, c_{k,i}^*\rangle$ for each arm $i\in[N]$ and cost type $k\in[K]$. 
Subsequently, $\EE\mbracket{\Delta_t\givenmiddle\mX_t}$ can be rewritten and bounded as 
\begin{align}
    \nonumber
     \EE\mbracket{\Delta_t\givenmiddle\mX_t} 
     &= \E{\maxk \max_{m\in[0,1]_N} \abs{\sum_{i\in[Nm]}\xi_{k,i}} \givenmiddle \mX_t} \\
     \nonumber
     &\leq  \sumk \EE\mbracket{\max_{n\in[N]}\abs{\sum_{i\in[n]}\xi_{k,i}} \givenmiddle \mX_t} \\
    \label{eq:pf-conformity:before-martingale}
     &\leq  \sumk \EE\mbracket{\bracket{\max_{n\in[N]} \abs{\sum_{i\in[n]}\xi_{k,i}}}^2 \givenmiddle \mX_t}^{1/2},
\end{align}
where \eqref{eq:pf-conformity:before-martingale} follows from the Cauchy–Schwarz inequality. 
Next, we argue that the sequence $\sets{\abs{\sum_{i\in[n]}\xi_{k,i}}}_{n\in[N]}$ is a submartingale, which allows us to apply Doob's $L_2$ inequality (\citealp[Theorem 5.4.3]{Dur_19_prob_book}) to bound the expression in \eqref{eq:pf-conformity:before-martingale}. Observe that, for each cost type $k$, conditioned on $\mX_t$, the sequence of random variables $\sets{\xi_{k,i}}_{i\in[N]}$ are independent and have zero conditional means. 
Consequently, the sequence of partial sums $\left\{\sum_{i\in[n]}\xi_{k,i}\right\}_{n\in[N]}$ forms a martingale, which becomes a submartingale upon taking the absolute value. 
Thus, by applying Doob's $L_2$ inequality and utilizing the bound $\xi_{k,i}\leq c_{\max}$, we obtain 
\begin{align}    
    \nonumber
    \EE\mbracket{\Delta_t\givenmiddle\mX_t} 
    &\leq \sumk 2\EE\mbracket{ \bigg(\sum_{i\in[N]}\xi_{k,i}\bigg)^2 \givenmiddle \mX_t}^{1/2} \\
    \nonumber
    &=  \sumk 2\bigg(\sum_{i\in[N]} \EE\mbracket{\xi_{k,i}^2 \givenmiddle \mX_t}\bigg)^{1/2}\\ 
    \label{eq:pf-conform:E-delta-bound}
    &\leq 2Kc_{\max}\sqrt{N}\,.
\end{align}
Combining \eqref{eq:pf-conform:E-delta-bound} with our earlier calculations, we get 
\begin{align*}
    \E{\big(Nm(\mX_t) - N_t^*\big)^+ \givenmiddle \mX_t} 
    \leq \EE\mbracket{\frac{\Delta_t+M_c}{\eta_c} \givenmiddle \mX_t} 
    \leq \frac{2Kc_{\max}}{\eta_c}\sqrt{N}+\frac{M_c}{\eta_c} \leq K_{\conf}\sqrt{N}. 
\end{align*}
where $K_{\conf} \triangleq (2Kc_{\max} + M_c) / \eta_c$. 
\end{proof}

\subsection{Proof of Lemma~\ref{lem:non-shrinking} (Almost non-shrinking)}

We first state and prove a supporting lemma below, which will be used in the proof of Lemma~\ref{lem:non-shrinking}.

\begin{lemma}\label{lem:focus set drift}
    Under the ID policy, we have 
    \begin{align*}
        \EE\mbracket{\big(h_{\ID}(\mX_{t+1},m(\mX_t))- \gamma h_{\ID}(\mX_{t},m(\mX_t)\big)^+\givenmiddle\mX_t}\leq \big(2 + K_{\conf}\big) C_h\sqrt{N}\,,
    \end{align*}
    where $C_h > 0$ is the positive constant given in \Cref{lem:drift}. 
\end{lemma}
\begin{proof}
    We upper bound $\EE\mbracket{\bracket{h_{\ID}(\mX_{t+1},m(\mX_t))-\gamma h_{\ID}(\mX_t,m(\mX_t))}^+\givenmiddle\mX_t}$ by coupling $\mX_{t+1}$ with a random element $\mX_{t+1}'$ constructed as follows: Let $\mX_{t+1}'$ be the system state at step $t+1$ if we were able to set $A_{i,t}=\widehat{A}_{i,t}$ for all $i\in[N]$. From the drift condition of the Lyapunov function $h_{\ID}(\cdot,D)$, as established in \Cref{lem:hID drift}, we obtain: 
    \begin{equation}
        \label{eq:pf-hid-x-mx-shrink:apply-drift-of-hid}
        \EE\mbracket{(h_{\ID}(\mX_{t+1}',m(\mX_t))-\gamma h_{\ID}(\mX_t,m(\mX_t)))^+\givenmiddle\mX_t} \leq 2C_h\sqrt{N}\,.
    \end{equation}
    We couple $\mX_{t+1}'$ and $\mX_{t+1}$ such that $\mX_{i,t+1}'=\mX_{i,t+1}$ for all $i\leq \min\sets{N_t^*,Nm(\mX_t)}$. Then we have
    \begin{align}
        \nonumber
        &\mspace{23mu} \EE\mbracket{\big(h_{\ID}(\mX_{t+1},m(\mX_t))-\gamma h_{\ID}(\mX_t,m(\mX_t))\big)^+-\big(h_{\ID}(\mX_{t+1}',m(\mX_t))-\gamma h_{\ID}(\mX_t,m(\mX_t))\big)^+\givenmiddle\mX_t}\\
        \nonumber
        &\leq \EE\mbracket{\big(h_{\ID}(\mX_{t+1},m(\mX_t))-h_{\ID}(\mX_{t+1}',m(\mX_t))\big)^+\givenmiddle\mX_t}\\
        \nonumber
        &\leq \EE\mbracket{\max_{m'\in[0,1]_N\colon m'\leq m(\mX_t)}\big(h(\mX_{t+1},[Nm'])-h(\mX_{t+1}',[Nm'])\big)^+\givenmiddle\mX_t}\\
        \nonumber
        &\leq \EE\mbracket{\max_{m'\in[0,1]_N\colon m'\leq m(\mX_t)} \max_{g\in \mathcal{G}}\supell \abs{\sum_{i\in[Nm']}\innerproduct{(X_{i,t+1}'-X_{i,t+1})P_i^{\ell} \gamma^{-\ell}}{g_i}}\givenmiddle\mX_t}\\
        \nonumber
        &= \EE\mbracket{\max_{m'\in[0,1]_N\colon m'\leq m(\mX_t)}\max_{g\in \mathcal{G}}\supell \abs{\sum_{i\in[Nm']}\innerproduct{(X_{i,t+1}'-X_{i,t+1})(P_i-\Xi_i)^{\ell} \gamma^{-\ell}}{g_i}}\givenmiddle\mX_t}\\
        \nonumber
        &= \EE\mbracket{\sum_{i\in[Nm(\mX_t)]\backslash[N_t^*]} \sum_{g\in \mathcal{G}}\sumell  \abs{\sum_{i\in[Nm']}\innerproduct{(X_{i,t+1}'-X_{i,t+1})(P_i-\Xi_i)^{\ell} \gamma^{-\ell}}{g_i}}\givenmiddle\mX_t}\\
        \nonumber
        &\leq \EE\mbracket{\sum_{i\in[Nm(\mX_t)]\backslash[N_t^*]} \sum_{g\in \mathcal{G}}\sumell \gamma^{-\ell} \norm{X_{i,t+1}-X_{i,t+1}'}_1\norm{(P_i-\Xi_i)^{\ell}}_\infty\norm{g_i}_\infty\givenmiddle\mX_t}\\
        \nonumber
        &\leq \EE\mbracket{\sum_{i\in[Nm(\mX_t)]\backslash[N_t^*]} 2 \sumell \gamma^{-\ell}\norm{(P_i-\Xi_i)^{\ell}}_\infty \sum_{g\in \mathcal{G}}\norm{g_i}_\infty \givenmiddle\mX_t}\\
        \nonumber
        &\leq \EE\mbracket{\sum_{i\in[Nm(\mX_t)]\backslash[N_t^*]}2\Ctau(Kc_{\max}+r_{\max})\givenmiddle\mX_t}\\
        \label{eq:pf-hid-x-mx-shrink:before-apply-conformity}
        &\leq 2\Ctau(Kc_{\max}+r_{\max}) \EE\mbracket{(Nm(\mX_t)-N_t^*)^+\givenmiddle\mX_t}\\
        \label{eq:pf-hid-x-mx-shrink:apply-conformity}
        &\leq 2 \Ctau(Kc_{\max}+r_{\max})K_{\conf}\sqrt{N}\,,
    \end{align}
    where we have applied \Cref{lem:majority_conformal} to bound the expression in \eqref{eq:pf-hid-x-mx-shrink:before-apply-conformity}.  
    
    By combining \eqref{eq:pf-hid-x-mx-shrink:apply-drift-of-hid} and \eqref{eq:pf-hid-x-mx-shrink:apply-conformity}, we obtain:  
    \begin{align*}
        \EE\mbracket{\bracket{h_{\ID}(\mX_{t+1},m(\mX_t))-\gamma h_{\ID}(\mX_t,m(\mX_t))}^+\givenmiddle\mX_t}&\leq \left(2C_h+ 2\Ctau(Kc_{\max}+r_{\max})K_{\conf}\right)\sqrt{N} \\
        &= \left(2C_h+ C_h K_{\conf}\right)\sqrt{N},
    \end{align*}
    where the last equality is because $C_h = 2 \Ctau(Kc_{\max} + r_{\max})$. 
\end{proof}

We now give the proof of Lemma~\ref{lem:non-shrinking}.
\begin{proof}[Proof of Lemma~\ref{lem:non-shrinking}]
    First, we claim that 
    \begin{align}\label{eq:lem8eq1}
        m(\mX_{t+1})\geq m(\mX_t) - \frac{1}{\eta_c N} \bracket{h_{\ID}(\mX_{t+1},m(\mX_{t}))-h_{\ID}(\mX_t,m(\mX_t))}^+ - \frac{M_c}{\eta_cN}.
    \end{align}
    To prove the claim, by the maximality of $m(\mX_{t+1})$ it suffices to show that for any $\bar{m}\in [0,1]_N$ s.t.
    \begin{equation}
        \label{eq:pf-non-shrink:mbar-claim}
        \Bar{m}\leq m(\mX_t) - \frac{1}{\eta_cN} \bracket{h_{\ID}(\mX_{t+1},m(\mX_t))-h_{\ID}(\mX_t,m(\mX_t))}^+ -\frac{M_c}{\eta_cN},
    \end{equation}
    we have $h_{\ID}(\mX_{t+1},\Bar{m})\leq \rembud([N\Bar{m}])$.
    For any $\Bar{m}$ satisfying \eqref{eq:pf-non-shrink:mbar-claim}, Lemma \ref{lem:positiveC} implies that
    \begin{align*}
        \rembud([N\Bar{m}])-\rembud([Nm(\mX_t)])
        &\geq \eta_c N \bracket{m(\mX_t) - \bar{m}} - M_c  \\
        &\geq\eta_c N\bracket{\frac{1}{\eta_c N}\bracket{h_{\ID}(\mX_{t+1},\Bar{m})-h_{\ID}(\mX_t,m(\mX_t))}^++\frac{M_c}{\eta_cN} } - M_c\\
        &\geq\bracket{h_{\ID}(\mX_{t+1},\Bar{m})-h_{\ID}(\mX_t,m(\mX_t))}^+ \,.
    \end{align*}
    Since $\rembud(m([N\mX_t)])\geq h_{\ID}(\mX_t,m(\mX_t))$ by the definition of $m(\mX_t)$, we thus have
    \begin{align*}
        \rembud([N\Bar{m}])&\geq \rembud([Nm(\mX_t)]) + \bracket{h_{\ID}(\mX_{t+1},\Bar{m})-h_{\ID}(\mX_t,m(\mX_t))}^+ \\
        &\geq h_{\ID}(\mX_t,m(\mX_t)) + \bracket{h_{\ID}(\mX_{t+1},\Bar{m})-h_{\ID}(\mX_t,m(\mX_t))}^+ \\
        &\geq h_{\ID}(\mX_{t+1},\Bar{m})\,,
    \end{align*}
    which proves the claim in \eqref{eq:lem8eq1}. 

    Taking the conditional expectations in \eqref{eq:lem8eq1} and rearranging the terms, we get
    \begin{align*}
       \EE\mbracket{(m(\mX_t)-m(\mX_{t+1}))^+\givenmiddle\mX_t}\leq \frac{1}{\eta_c N}\EE\mbracket{\bracket{h_{\ID}(\mX_{t+1},m(\mX_t))-h_{\ID}(\mX_t,m(\mX_t))}^+\givenmiddle\mX_t} + \frac{M_c}{N\eta_c}, 
    \end{align*}
    where the right-hand side can be further bounded using \Cref{lem:focus set drift}, which states that
    \begin{align*}
        \EE\mbracket{\bracket{h_{\ID}(\mX_{t+1},m(\mX_t))-h_{\ID}(\mX_t,m(\mX_t))}^+\givenmiddle\mX_t}\leq \big(2 + K_{\conf}\big) C_h\sqrt{N}\,.
    \end{align*}
    Therefore, we have
    \begin{align*}
        \EE\mbracket{(m(\mX_t)-m(\mX_{t+1}))^+\givenmiddle\mX_t}\leq  \frac{\big(2 + K_{\conf}\big) C_h}{\eta_c \sqrt{N}}+\frac{M_c}{\eta_c N},
    \end{align*}
    which implies 
    \[
        \EE\mbracket{(m(\mX_t)-m(\mX_{t+1}))^+\givenmiddle\mX_t} \leq \frac{K_{\mono}}{\sqrt{N}}, 
    \]
    with $K_{\mono} \triangleq \Big(\big(2 + K_{\conf}\big) C_h  + M_c \Big) / \eta_c$. 
\end{proof}

\subsection{Proof of Lemma~\ref{lem:coverage} (Sufficient coverage)}
\begin{proof}
    Observe that it suffices to focus on the case when $m(\mX_t) \neq 1$. 
    Recall that for any system state $\mx$, $m(\mx)$ is defined as 
    \begin{align}\tag{\ref{eq:focus set def}}
        m(\mx)= \max \sets{ m\in[0,1]_N: h_{\ID}(\mx,m)\leq \rembud([Nm])}. 
    \end{align}
    Because $m(\mX_t) \neq 1$, we have $m(\mX_t)+1/N  \in [0,1]_N$. Then the maximality of $m(\mX_t)$ implies that
    \begin{equation}
    \label{eq:lem9eq1}
        h_{\ID}(\mX_t,m(\mX_t)+1/N) > \rembud([Nm(\mX_t)+1]). 
    \end{equation}
    We can upper bound the left-hand side of \eqref{eq:lem9eq1} using the Lipschitz continuity of $h$:
    \begin{align}\label{eq:lem9eq2}
        h_{\ID}(\mX_t,m(\mX_t)+1/N)\leq h_{\ID}(\mX_t,m(\mX_t)) + L_h.
    \end{align}
    We then lower bound the right-hand side of \eqref{eq:lem9eq1} using Lemma \ref{lem:positiveC}: 
    \begin{align}
        \nonumber
        \rembud([Nm(\mX_t)+1])
        &\geq  \rembud([Nm(\mX_t)+1]) - \rembud([N])\\ 
        \nonumber
        &\geq \eta_c(N-Nm(\mX_t)-1)-M_c \\
        \label{eq:lem9eq3}
        &=\eta_c N(1-m(\mX_t))-\eta_c - M_c. 
    \end{align}
    Comparing \eqref{eq:lem9eq1}, \eqref{eq:lem9eq2} and \eqref{eq:lem9eq3}, we have:
    \begin{align*}
        h_{\ID}(\mX_t,m(\mX_t))\geq \eta_c N(1-m(\mX_t))-\eta_c - M_c - L_h\,,
    \end{align*}
    which, after rearranging the terms, implies
    \begin{align*}
        1-m(\mX_t)\leq \frac{1}{\eta_c N}h_{\ID}(\mX_t,m(\mX_t))+\frac{K_{\cov}}{N}, 
    \end{align*}
    with $K_{\cov} \triangleq (\eta_c+M_c+L_h)/\eta_c$.
\end{proof}

\section{Proofs of Lemma~\ref{lem:optimality gap} and Lemma~\ref{lem:drift-V}}
\label{sec:proof-lem-main-theorem}

In this section, we provide two final lemmas, \Cref{lem:optimality gap} and \Cref{lem:drift-V}, which together imply \Cref{thm:opt-gap-bound}. We prove these two lemmas in Sections~\ref{app:pf-lem-opt-gap} and \ref{app:pf-lem:drift-V}, respectively. 

\subsection{Proof of Lemma~\ref{lem:optimality gap}}
\label{app:pf-lem-opt-gap}

\optgapbyV*

\begin{proof}
    We can bound the optimality gap as the following long-run average: 
    \begin{align}
        \nonumber
        R^*(N,\bm{S}_0) - R(\pi,\bm{S}_0) 
        &\leq R^{\rel}(N,\bm{S}_0)-R(\pi,\bm{S}_0)\\
        \nonumber
        &= R^{\rel}(N,\bm{S}_0) - \lim_{T\to\infty} \frac{1}{T} \sum_{t=0}^{T-1} \frac{1}{N}\sumN \EE[r_i(S_{i,t},A_{i,t})]  \\
        \nonumber
        &= \lim_{T\to\infty} \frac{1}{T} \sum_{t=0}^{T-1} \Big(R^{\rel}(N,\bm{S}_0) - \frac{1}{N}\sumN \EE[r_i(S_{i,t},A_{i,t})]\Big).
    \end{align}
    To bound $R^{\rel}(N,\bm{S}_0) - \sumN \EE[r_i(S_{i,t},A_{i,t})] / N$, we calculate that
    \begin{align}
        \nonumber
        &\mspace{23mu} R^{\rel}(N,\bm{S}_0) - \frac{1}{N}\sumN \EE[r_i(S_{i,t},A_{i,t})]  \\
        \nonumber
        &= \frac{1}{N}\sumN\sum_{s\in\sspa,a\in\aspa}r_i(s,a)y_i^*(s,a) - \frac{1}{N}\sumN \EE[r_i(S_{i,t}, A_{i,t})] \\
        \nonumber
        &\leq \frac{1}{N}\sumN\sum_{s\in\sspa,a\in\aspa}r_i(s,a)y_i^*(s,a) - \frac{1}{N}\sumN \EE[r_i(S_{i,t},\widehat{A}_{i,t})] + \frac{2r_{\max}}{N}\sumN\EE\mbracket{\mathbf{1}\sets{A_{i,t}\neq\widehat{A}_{i,t}}}\\
        \nonumber
        &= \frac{1}{N}\sumN\innerproduct{\mu^*_i-\EE[X_{i,t}]}{r_i^*}  + \frac{2r_{\max}}{N}\sumN\EE\mbracket{\mathbf{1}\sets{A_{i,t}\neq\widehat{A}_{i,t}}} \\
        \label{eq:pf-opt-gap-bd:eq-0}
        &\leq \frac{1}{N}\sumN\innerproduct{\mu^*_i-\EE[X_{i,t}]}{r_i^*} + 2r_{\max} \EE\mbracket{1-\frac{N_{t}^*}{N}}\\
        \label{eq:pf-opt-gap-bd:eq-1}
        &\leq
        \frac{1}{N}\sumN\innerproduct{\mu^*_i-\EE[X_{i,t}]}{r_i^*} + 2r_{\max} \EE\mbracket{1-m(\mX_{t})} + \frac{2r_{\max} K_{\conf}}{\sqrt{N}} \\
        \label{eq:pf-opt-gap-bd:eq-2}
        &\leq \frac{1}{N}\EE\mbracket{h_{\ID}(\mX_{t},1)} + 2r_{\max} \EE\mbracket{1-m(\mX_{t})} + \frac{2r_{\max} K_{\conf}}{\sqrt{N}},
    \end{align}
    where \eqref{eq:pf-opt-gap-bd:eq-0} is due to the definition of $N_t^*$; \eqref{eq:pf-opt-gap-bd:eq-1} is due to the bound on $\EE[\bracket{Nm(\mX_t) - N_t^* }^+]$ in \Cref{lem:majority_conformal}, and \eqref{eq:pf-opt-gap-bd:eq-2} directly follows from the definition of $\hid(\mx, 1)$. 
    We further bound the first two expressions in \eqref{eq:pf-opt-gap-bd:eq-2} in terms of $V(\mX_t)$ as follows: 
    \begin{align}
        \label{eq:pf-opt-gap-bd:eq-4}
        h_{\ID}(\mX_{t},1)&\leq h_{\ID}(\mX_{t},m(\mX_{t}))+L_h N (1-m(\mX_{t}))=V(\mX_{t}) \\
        \label{eq:pf-opt-gap-bd:eq-3}
        1-m(\mX_{t}) &\leq \frac{1}{L_h N}V(\mX_{t}),
    \end{align}
    where \eqref{eq:pf-opt-gap-bd:eq-4} is due to the Lipschitz continuity of $\hid(\mx, m)$ with respect to the parameter $m$ (Lemma~\ref{lem:hID drift}), and \eqref{eq:pf-opt-gap-bd:eq-3} is due to the definition of $V(\mx)$.  

    Combining the above calculations, we get
    \begin{align*}
       R^*(N,\bm{S}_0) - R(\pi,\bm{S}_0)
       &\leq  \lim_{T\to\infty} \frac{1}{T} \sum_{t=0}^{T-1} \Big(\frac{1}{N}\EE\mbracket{h_{\ID}(\mX_{t},1)} + 2r_{\max} \EE\mbracket{1-m(\mX_{t})} + \frac{2 r_{\max} K_{\conf}}{\sqrt{N}}   \Big) \\
       &\leq \frac{2r_{\max} + L_h}{L_h N}\lim_{T\to\infty}\frac{1}{T}\sum_{t=0}^{T-1}\EE\mbracket{V(\mX_{t})} + \frac{2 r_{\max} K_{\conf}}{\sqrt{N}}\,. \qedhere
    \end{align*}
\end{proof}

\subsection{Proof of Lemma~\ref{lem:drift-V}}
\label{app:pf-lem:drift-V}

\driftV*

\begin{proof}
We derive a recurrence relation between $\E{V(\mX_{t+1})}$ and $\E{V(\mX_t)}$, by bounding $\E{V(\mX_{t+1}) \givenplain \mX_t} - V(\mX_t)$. Specifically, observe that by the Lipschitz continuity of $\hid(\mX, D)$ with respect to $D$, we have
    \begin{align*}
        V(\mX_{t+1})
        &= h_{\ID}(\mX_{t+1},m(\mX_{t+1})) + L_hN(1-m(\mX_{t+1})) \\
        &\leq h_{\ID}(\mX_{t+1},m(\mX_t)) + L_hN(1-m(\mX_t)) + 2L_hN(m(\mX_t)-m(\mX_{t+1}))^+\,.
    \end{align*}
    Consequently, we have
    \begin{align}
        \nonumber
        &\EE\mbracket{V(\mX_{t+1}) \mid \mX_t}-V(\mX_t)\\
        \nonumber
        &\leq\EE\mbracket{h_{\ID}(\mX_{t+1},m(\mX_t)) \mid \mX_t}- h_{\ID}(\mX_t,m(\mX_t)) + 2L_hN\EE\mbracket{(m(\mX_t)-m(\mX_{t+1}))^+ \mid \mX_t}\\
        \label{eq:pf-thm:eq-1}
        &\leq \EE\mbracket{h_{\ID}(\mX_{t+1},m(\mX_t)) \mid \mX_t}- h_{\ID}(\mX_t,m(\mX_t)) + 2L_hK_{\mono}\sqrt{N}\,,
    \end{align}
    where the last inequality follows from \Cref{lem:non-shrinking}. To bound $\EE\mbracket{h_{\ID}(\mX_{t+1},m(\mX_t)) \mid \mX_t}$, observe that by \Cref{lem:majority_conformal}, all but $O(\sqrt{N})$ arms in $[Nm(\mX_t)]$ follow the optimal single-armed policies, so the drift condition of $\hid$ applies to this set of arms. As formalized in \Cref{lem:focus set drift}, we can thus show that
    \begin{align}
        \EE[h_{\ID}(\mX_{t+1},m(\mX_t))\givenplain \mX_t ]\leq \gamma h_{\ID}(\mX_t,m(\mX_t)) + \big(2 + K_{\conf}\big) C_h\sqrt{N}\,.\label{eq:pf-thm:invoke-focus-set-drift}
    \end{align}
    Plugging \eqref{eq:pf-thm:invoke-focus-set-drift} back to \eqref{eq:pf-thm:eq-1}, we get 
    \begin{equation}
    \label{eq:pf-thm:eq-2}
    \begin{aligned}
        \EE\mbracket{V(\mX_{t+1}) \mid \mX_t}-V(\mX_t)
        &\leq  -(1-\gamma) h_{\ID}(\mX_t,m(\mX_t))  \\
        &\mspace{23mu} + \left(2C_h+C_h K_{\conf} + L_hK_{\mono}\right)\sqrt{N}. 
    \end{aligned}
    \end{equation}
    To further bound $h_{\ID}(\mX_t,m(\mX_t))$ in \eqref{eq:pf-thm:eq-2} in terms of $V(\mX_t)$, we apply \Cref{lem:coverage} to get: 
    \begin{align}
        \nonumber
        V(\mX_t) &= h_{\ID}(\mX_t,m(\mX_t)) + L_hN(1-m(\mX_t))\\
        &\leq \Big(1+\frac{L_h}{\eta_c}\Big) h_{\ID}(\mX_t,m(\mX_t)) + L_hK_{\cov}\,. \label{eq:pf-thm:invoke-sufficient-coverage}
    \end{align}
    Substituting \eqref{eq:pf-thm:invoke-sufficient-coverage} into \eqref{eq:pf-thm:eq-2} and rearranging the terms, we get: 
    \begin{equation}
        \label{eq:pf-thm:drift-inequality}
        \EE\mbracket{V(\mX_{t+1}) \givenplain \mX_t}-V(\mX_t) \leq -\rhov V(\mX_t) + \Kv \sqrt{N}, 
    \end{equation}
    where $\rhov = (1-\gamma) / (1+\frac{L_h}{\eta_c})$, and 
    \[
        \Kv = 2C_h +C_h K_{\conf}+2L_hK_{\mono}+ \frac{\rhov L_h K_{\cov}}{\sqrt{N}}. 
    \]
    
    Taking the expectation in \eqref{eq:pf-thm:drift-inequality} and unrolling the recursion of $\E{V(\mX_t)}$, we get
    \begin{equation*}
        \E{V(\mX_{t})} \leq (1-\rhov)^t \E{V(\mX_0)} + \frac{\Kv\sqrt{N}}{\rhov}. 
    \end{equation*}
    Averaging over $T$ time steps, we have 
    \begin{equation}
        \label{eq:pf-thm:final-EV-bound-finite}
        \frac{1}{T} \sum_{t=0}^{T-1} \E{V(\mX_{t})} \leq \frac{\E{V(\mX_0)}}{\rhov T} + \frac{\Kv\sqrt{N}}{\rhov}.
    \end{equation}
    Taking the limit of $T\to\infty$, we can bound the long-run-averaged expectation of $\E{V(\mx)}$ as 
    \begin{equation}
        \label{eq:pf-thm:final-EV-bound}
        \lim_{T\to\infty} \frac{1}{T} \sum_{t=0}^{T-1} \E{V(\mX_{t})} \leq \frac{\Kv\sqrt{N}}{\rhov},
    \end{equation}
    which completes the proof. \qedhere
    \end{proof}

    By examining the proof, we can get a more explicit form of the constant $C_{\ID}$ in Theorem~\ref{thm:opt-gap-bound}: 
    \begin{itemize}%
        \item If $\cmax < \alpha_{\min}$, the budget constraints are vacuous and all arms will follow the ideal actions under the ID policy. Consequently, $R^*(N,\bm{S}_0) = R(\pi,\bm{S}_0)$ and we can simply take $C_{\ID} = 0$. 
        \item If $\cmax \geq \alpha_{\min}$, we combine the optimality gap bound in \Cref{lem:optimality gap} with \eqref{eq:pf-thm:final-EV-bound} to get
        \begin{equation*}
            R^*(N,\bm{S}_0)-R(\pi,\bm{S}_0)\leq \frac{(2r_{\max}+L_h)\Kv}{L_h \rhov \sqrt{N}} + \frac{2r_{\max} K_{\conf}}{\sqrt{N}}. 
        \end{equation*}
        We therefore obtain $R^*(N,\bm{S}_0)-R(\pi,\bm{S}_0) \leq C_{\ID} / \sqrt{N}$, where $C_{\ID}$ is given by
        \begin{align*}
             C_{\ID}&=\frac{\bracket{2\rmax + L_h}\bracket{1/L_h + 1/\eta_c}}{1-\gamma} \bracket{2C_h+ C_h K_{\conf}+2L_hK_{\mono}} \\
             &\mspace{23mu} + (2\rmax+L_h)K_{\cov} + 2 r_{\max} K_{\conf} \\
             &= O\left(\frac{K^5 \max\{r_{\max}, c_{\max}\}^7 \tau^4}{\alpha_{\min}^6}\right). 
        \end{align*} 
        In the last step above, we substituted in all intermediate constants from the proofs; one can verify that each of these constants depends on the problem parameters as follows: 
        \begin{align*}
            L_h &= O\big(\max\{\rmax, \cmax\}\tau\big) \\
            C_h &= O\big(K \max\{\rmax, \cmax\}\tau\big) \\
            \frac{1}{1-\gamma} &= O(\tau) \\
            \frac{1}{\eta_c} &= O\left(\frac{K \cmax}{\alpha_{\min}^2}\right) \\
            K_{\conf} &= O\left(\frac{K^2 \cmax^2}{\alpha_{\min}^2}\right) \\
            K_{\mono} &= O\left(\frac{K^4 \cmax^3 \max\{\rmax, \cmax\}}{\alpha_{\min}^4}\tau \right) \\
            K_{\cov} &= O\left(\frac{K \cmax \max\{\rmax, \cmax\}}{\alpha_{\min}^2} \tau\right). 
        \end{align*}
    \end{itemize}

\section{Proof of \Cref{prop:finite-time-bound}}
\label{app:finite-time-bound-pf}
    Now we briefly outline the proof of \Cref{prop:finite-time-bound}, which states that under the same conditions as \Cref{thm:opt-gap-bound}, for any $T\geq 1$, we have
    \begin{equation}
        \label{eq:prop-finite-time-bound-restate}
        R^*(N, \bm{S}_0) - \frac{1}{TN}\sum_{t=0}^{T-1}\sum_{i\in[N]}\E{r_i(S_{i,t}^\pi, A_{i,t}^\pi)} \leq \frac{C_\ID}{\sqrt{N}} + \frac{C_\finite}{T},
    \end{equation}
    where $C_\finite$ is another positive constant independent of $N$. 
    This follows from two steps:
    \begin{itemize}
    \item First, it is not hard to adapt the proof of \Cref{lem:optimality gap} in \Cref{app:pf-lem-opt-gap} to show that
    \begin{equation}
        \label{eq:pf-finite-time-bound-step-1}
        \begin{aligned}
            &R^*(N, \bm{S}_0) - \frac{1}{TN}\sum_{t=0}^{T-1}\sum_{i\in[N]}\E{r_i(S_{i,t}^\pi, A_{i,t}^\pi)} \\ 
            &\qquad \leq \left(1 + \frac{2r_{\max}}{L_h}\right) \frac{1}{T N}\sum_{t=0}^{T-1} \mathbb{E}[V(\bm{X}_t)] + \frac{2r_{\max}K_{\text{conf}}}{\sqrt{N}}.
        \end{aligned}
    \end{equation}
    Specifically, substituting \eqref{eq:pf-opt-gap-bd:eq-3} and 
    \eqref{eq:pf-opt-gap-bd:eq-4} into \eqref{eq:pf-opt-gap-bd:eq-2} implies that for any $t\geq 0$, we have
    \begin{equation}
        \rrel - \frac{1}{N}\sum_{i\in[N]}\E{r_i(S_{i,t}^\pi, A_{i,t}^\pi)}  \leq \left(1 + \frac{2r_{\max}}{L_h}\right) \frac{1}{N} \mathbb{E}[V(\bm{X}_t)] + \frac{2r_{\max}K_{\text{conf}}}{\sqrt{N}}.
    \end{equation}
    Averaging the above inequality over $t=0,1,2,\dots T-1$ and applying $R^*(N, \bm{S}_0) \leq \rrel$, we have \eqref{eq:pf-finite-time-bound-step-1}. 
    \item Second, we apply \eqref{eq:pf-thm:final-EV-bound-finite} in the proof of \Cref{lem:drift-V} to get
    \begin{align}
        \nonumber
            &R^*(N, \bm{S}_0) - \frac{1}{TN}\sum_{t=0}^{T-1}\sum_{i\in[N]}\E{r_i(S_{i,t}^\pi, A_{i,t}^\pi)}\\
        \nonumber
            &\qquad \leq \left(1 + \frac{2r_{\max}}{L_h}\right) \left(\frac{\E{V(\mX_0)}}{\rhov TN} + \frac{\Kv}{\rhov\sqrt{N}}\right) + \frac{2r_{\max}K_{\text{conf}}}{\sqrt{N}} \\
        \nonumber
            &\qquad = \left(1 + \frac{2r_{\max}}{L_h}\right) \frac{\E{V(\mX_0)}}{\rhov TN} + \frac{C_{\ID}}{\sqrt{N}} \\
            &\qquad \leq  \left(1 + \frac{2r_{\max}}{L_h}\right) \frac{L_h}{\rhov T} + \frac{C_{\ID}}{\sqrt{N}}.
    \end{align}
    where the last inequality is because of $V(\mX_0) \leq N L_h$ for any system state $\mX_0$. Letting $C_\finite =  \left(1 + \frac{2r_{\max}}{L_h}\right) \frac{L_h}{\rhov}$, we get \eqref{eq:prop-finite-time-bound-restate}. 
    \end{itemize}

\end{document}